\documentclass{article}

\usepackage{arxiv}

\usepackage[utf8]{inputenc} 
\usepackage[T1]{fontenc}    
\usepackage{hyperref}       
\usepackage{url}            
\usepackage{booktabs}       
\usepackage{amsfonts}       		
\usepackage{graphicx}
\usepackage{amsmath}
\usepackage{amsthm}
\usepackage{algorithm}
\usepackage{algorithmic}
\usepackage{orcidlink}

\newtheorem{definition}{Definition}
\newtheorem{lemma}{Lemma}
\newtheorem{proposition}{Proposition}
\newtheorem{theorem}{Theorem}

\title{The Push-Forward Transform for Continuous and Robust Comparison of Dynamic Shapes}
\author{
Roua Rouatbi\,\orcidlink{0009-0004-9010-677X}$^{1,2,3,4}$,
Juan-Esteban Suarez Cardona\,\orcidlink{0000-0002-1990-2413}$^{6,7}$,
\And
Ivo F.~Sbalzarini\,\orcidlink{0000-0003-4414-4340}$^{1,2,3,4,5,8}$\\
\\
$^1$ Faculty of Computer Science, Dresden University of Technology, Dresden, Germany \\
$^2$ Max Planck Institute of Molecular Cell Biology and Genetics, Dresden, Germany \\
$^3$ Center for Systems Biology Dresden, Dresden, Germany \\
$^4$ Center for Scalable Data Analytics and Artificial Intelligence (ScaDS.AI) Dresden/Leipzig \\
$^5$ Cluster of Excellence Physics of Life, Dresden University of Technology, Dresden, Germany \\
$^6$ Chair for Mathematical Foundations of Artificial Intelligence, Ludwig-Maximilians-Universit\"at M\"unchen, Munich, Germany \\
$^7$ Munich Center for Machine Learning (MCML), Munich, Germany \\
$^8$ Now at: Department of Mathematical Modeling and Machine Learning, University of Zurich, Zurich, Switzerland
}

\hypersetup{
pdftitle={PF-SDMorphometric},
pdfauthor={Roua Rouatbi, Juan-Esteban Suarez Cardona, Ivo F.~Sbalzarini},
pdfkeywords={Geometric shape analysis, level-set methods, morphometric, shape quantification, invariant shape descriptors, geometric learning},
}

\date{}

\begin{document}

\maketitle

\begin{abstract}
We introduce a mathematical framework for shape comparison based on mapping functions from the shape domain to a common reference domain. This \textit{Push-Forward Transform} enables invariant and robust comparison of shapes, preserving intrinsic geometric information. Quantitatively comparing shapes and their temporal evolution is a fundamental challenge in image analysis. Meaningful shape comparison requires representations that are invariant to transformations that do not alter shape itself, such as translation, rotation, reflection, re-parametrization, and uniform scaling, while remaining sensitive to intrinsic geometric variation. Existing approaches often rely on sensitive parameterizations, landmark correspondence, or learned representations that are difficult to interpret and reproduce. We show that the Push-Forward Transform (PF-T) applied to Signed Distance Functions (SDFs) yields a continuous representation that captures both boundary and interior geometry. We derive an interpretable morphometric that quantifies shape similarity and reveals features such as skeletal topology and rotational symmetries. The push-forward transform applies consistently to two- and three-dimensional shapes, extends to time-evolving geometries, and supports the joint analysis of shape and additional scalar fields defined over shapes, such as intensity or molecular signals. We present the mathematical formulation, describe an efficient algorithm, and benchmark the approach on 2D, 3D, and temporal data sets.
\end{abstract}

\section{Introduction}
\label{sec:introduction}

The comparison of shapes is a fundamental problem in computer vision, medical imaging, and quantitative biology, where shape often serves as the primary descriptor of structure and function. A central challenge is to quantify shape differences in a way that is invariant to transformations that do not alter shape itself, such as translation, rotation, reflection, re-parametrization, and uniform scaling, while remaining sensitive to intrinsic geometric variation.

A natural strategy is to associate each shape with one or more functions defined over its domain, such as distance functions, curvature measures, or spatially varying signals, and to compare these functions across shapes. This approach, however, runs into a fundamental difficulty: functions defined on different shape domains are not directly comparable. Even when two shapes are topologically equivalent, their domains might differ geometrically, so that point-wise comparison is ill-defined and sensitive to arbitrary choices of parameterization or alignment.
Here, we address this comparability problem by introducing the \emph{Push-Forward Transform} (PF-T), which maps any function defined on a shape onto a common reference domain of identical topology through a smooth, structure-preserving diffeomorphism. The common reference domain enables well-defined comparison of distributions over shapes and yields metrics that are shape-preserving by construction. 

The PF-T can be applied to any continuous scalar-valued function over a shape. Of particular interest for deriving geometric descriptors is the \emph{signed distance function} (SDF) to the shape boundary~\cite{Osher2003-ga}. SDFs are (i) well-defined for any closed shape, (ii) encode both boundary location and interior structure in a single scalar field, and (iii) admit a smooth, computable approximation through the viscous Eikonal equation. Transforming the SDF of a shape via the PF-T yields the \emph{Push-Forward Signed Distance Function} (PF-SDF), a continuous and differentiable representation in which shapes can be compared invariantly~\cite{pfsdm}.

We demonstrate the utility of this framework for a range of applications. We show how our computational implementation with the viscous Eikonal equation enables the extraction of intrinsic skeletal structures, providing meaningful geometric shape descriptors that can be further used to construct parametrized reference axes. From the PF-SDF representation, we derive a spectral morphometric, the \emph{Push-Forward Signed Distance Morphometric} (PF-SDM), which is interpretable, invariant to shape-preserving transformations, and robust to noise, as we show in extensive benchmarks. Finally, we apply the PF-T to arbitrary intensity fields defined over shapes, enabling the joint modeling of morphology and spatial patterns as well as their evolution over time. This unified treatment of geometry and spatial signal significantly broadens the applicability of the proposed approach.

The main contributions of this work are:
\begin{itemize}
     \item We develop the theoretical foundations of the \emph{Push-Forward Transform} (PF-T) as a continuous and differentiable framework for comparing shapes and functions defined on them, yielding descriptors that are invariant by construction to translation, rotation, reflection, re-parametrization, and uniform scaling.

    \item We developed efficient approximations of the Signed Distance Function using polynomial surrogates to solve the viscous Eikonal learning problem, yielding a smooth representation from which intrinsic skeletal descriptors and parametrized reference axes (spines) of shapes can be extracted.

    \item We derive a spectral morphometric that provides interpretable and noise-robust shape descriptors capturing meaningful geometric features, such as rotational symmetries. We demonstrate its applicability to 2D shapes, 3D shapes, and time-evolving (2D\,+\,time) shapes.

    \item We extend the push-forward framework beyond geometry to jointly analyze shape and distributions defined over shapes, such as intensity or molecular signals, enabling unified modeling of morphology and spatial patterns across space and time.
\end{itemize}

This paper is structured as follows: In Section~\ref{sec:previous_work}, we review related work in geometric morphometrics and shape analysis. In Section~\ref{sec:push_forward_transform}, we present the PF-T, discuss its mathematical properties, and describe alternative choices of deformation maps. Section~\ref{sec:signed_distance_function} applies the PF-T to SDFs, introduces the viscous Eikonal formulation, and uses it to extract medial axes and parameterized spines. We present experimental benchmark results in Section~\ref{sec:experiments} for various applications, including rotational-symmetry-aware morphometrics, and the fusion of geometric and intensity signal features to predict the fate of time-evolving biological structures. In Section~\ref{sec:conclusion}, we conclude with a summary of findings and directions for future research.

\section{Related Work}\label{sec:previous_work}

Existing approaches to shape analysis and comparison can be broadly categorized by how shapes are represented and compared, and by whether invariance and interpretability are built into the representation or learned from data.

Traditional morphometric methods rely on explicit geometric representations and offer strong interpretability. Landmark-based approaches such as Generalized Procrustes Analysis (GPA)~\cite{Gower1975} compare shapes by aligning key points, hence factoring out translation, rotation, and scale. Boundary-based spectral descriptors, including Elliptical Fourier Analysis (EFA)~\cite{Kuhl1982} and spherical harmonics~\cite{Wieczorek2018}, encode shapes using truncated spectral bases defined on boundaries or surfaces, respectively. These methods are mathematically grounded and computationally efficient but depend on reliable parameterizations or landmark correspondences. In particular, spherical harmonic representations require a non-trivial spherical parameterization, which becomes challenging for shapes that are not star-convex.

An alternative line of work uses global region properties and moment-based descriptors such as aspect ratio, eccentricity, and compactness to summarize shape geometry in a low-dimensional, inherently invariant form. While these features are simple and fast to compute, they capture only coarse geometric characteristics and lack the expressiveness needed to distinguish nuanced morphological differences. Moreover, the choice of which features to use is itself non-trivial, leading back to the original question of what constitutes a meaningful comparison.

Learning-based approaches infer shape representations directly from data and have shown strong performance on large benchmarks. Variational autoencoder models such as ShapeEmbed~\cite{Annafoix} and O2VAE~\cite{Burgess2024}, contrastive learning frameworks (e.g., SimCLR~\cite{simclr}), and transformer-based models such as masked autoencoders (MAE)~\cite{MAE} are among the many methods applied to shape analysis. ShapeEmbed and O2VAE are particularly notable for their treatment of rotation invariance: O2VAE enforces it architecturally through O(2)-equivariant convolutional layers, while ShapeEmbed achieves it at the representation level by encoding shapes as distance matrices between boundary points. Most learning-based approaches, however, rely on large training data sets and offer limited interpretability of the resulting latent dimensions.

Unlike learning-based approaches, our proposed method is fully deterministic and does not require training data, yielding reproducible and interpretable results. At the same time, by operating on continuous functions mapped to a common reference domain, it overcomes key limitations of classical morphometric techniques, such as reliance on landmark correspondences or fragile shape parameterizations. Moreover, our framework naturally extends to additional functions defined over shapes and to time-varying dynamic shapes. This is made possible by the PF-T, as introduced in the next section.

\section{Push-Forward Transform}\label{sec:push_forward_transform}

The \emph{Push-Forward Transform} (PF-T) is the central mathematical concept for our framework. It maps functions defined over different shape domains to a common reference domain of identical topology, rendering a well-defined notion of comparability. This common-domain representation provides the basis for constructing descriptors that are invariant to shape-preserving transformations by construction.
We start by mathematically defining the PF-T and then describe practical constructions, including radial and harmonic extensions, as well as their approximations.
We denote by $\Omega\subseteq\mathbb{R}^d$ an open, bounded, Lipschitz domain, and by $\partial S\subseteq\Omega$ a shape boundary, assumed to be a smooth and closed co-dimension-one manifold embedded in $\Omega$.
Let \( f_S : S \rightarrow \mathbb{R} \) be a scalar field. The PF-T maps \( f_S \) to its push-forward \( f_{S^*} : S_r \rightarrow \mathbb{R} \), defined on a reference domain \( S_r \subseteq (-1,1)^d \), via a diffeomorphism \( \Psi : S_r \rightarrow S \). To formally define the PF-T, we start by recalling the definition of a level-set function:

\begin{definition}[Level-Set Function]\label{def:level_set}
Let $\Omega \subset \mathbb{R}^d$ and let $\partial S \subseteq \Omega$ be an oriented manifold of co-dimension one. Let $L_S : \Omega \to \mathbb{R}$ be a $C^1$ function such that $\nabla L_S \neq 0$ on $\partial S$.

We call $L_S$ a \emph{level-set function} associated with $S$ if and only if 
\[
\partial S = \{ x \in \Omega : L_S(x) = 0 \}.
\]

The interior and exterior domains of the shape $\partial S$ are:
\[
S := \{ x \in \Omega : L_S(x) > 0 \}, \
S^c := \{ x \in \Omega : L_S(x) < 0 \}.
\]
The closure of $S$ is given by $\overline{S}  =S\cup \partial S$.
\end{definition}

Following this, we define the PF-T.

\begin{definition}[Push-Forward Transform]\label{def:def_map}
Let $\partial S \subseteq \Omega$ be a shape with shape domain $S \subseteq \Omega$, and let 
$\partial S_r \subseteq \Omega$ be a reference shape with shape domain $S_r \subseteq \Omega$.
Furthermore, let $\Psi : \overline{S_r} \to \overline{S}$ be a diffeomorphism and $f_S : S\to \mathbb{R}$ a scalar field defined on the shape domain $S$. We define the \emph{Push-Forward Transform} (PF-T) $f_{S^*} : S_r \to \mathbb{R}$ as
\[
f_{S^*}(x) := f_S(\Psi(x)), \qquad \forall x \in S_r.
\]
\end{definition}

In general, determining a smooth and continuous deformation map between two shape domains is a nontrivial task, as it should preserve the geometric structure of the scalar fields being transformed. To address this, we first construct the boundary deformation map $\psi_S : \partial S_r \to \partial S$ and then choose an extension $\Psi: \overline{S_r} \to \overline{S}$, such that $\Psi|_{\partial S_r}= \psi_S$.

\begin{definition}[Boundary Deformation Map]
Let $\partial S, \partial S_r \subseteq \Omega$ be the target and reference shapes, respectively. The \emph{boundary deformation map} $\psi_S : \partial S_r \to \partial S$ is a diffeomorphism approximating the closest-point projection 
\begin{equation}\label{eq:def_bdy}
\inf_{\psi_S \in C^1(\partial S_r, \partial S)} 
\int\limits_{\partial S_r} \|\psi_S(x) - x_S(x)\|_2^2 \, \mathrm{d}\mathcal{H}(x)\, ,    
\end{equation}
where $x_S(x)$ denotes the closest-point projection of $x \in \partial S_r$ onto $\partial S$ and $\mathcal{H}$ is the Hausdorff-measure.
\end{definition}
To extend the boundary deformation map over the whole shape domain, different choices of extensions are possible, as we discuss next.

\subsection{Deformation Map Extensions}
We describe two possible extensions of the boundary deformation map $\psi_S$ for a given shape $\partial S$, namely the \emph{radial} and \emph{harmonic} extensions. For shapes that are topological spheres, we use the unit sphere $\partial B_r:=\{x\in \mathbb{R}^d:||x||_2=1\}$, with interior $B_r:=\{x\in \mathbb{R}^d:||x||_2<1\}$, as the reference shape. While other reference shapes can be used for shapes of different homology groups, this choice provides a simple domain that admits natural radial and harmonic extensions from the boundary to the interior. We start by defining the radial extension:

\begin{figure}[!t]
\centering
\includegraphics[width=3.5 in]{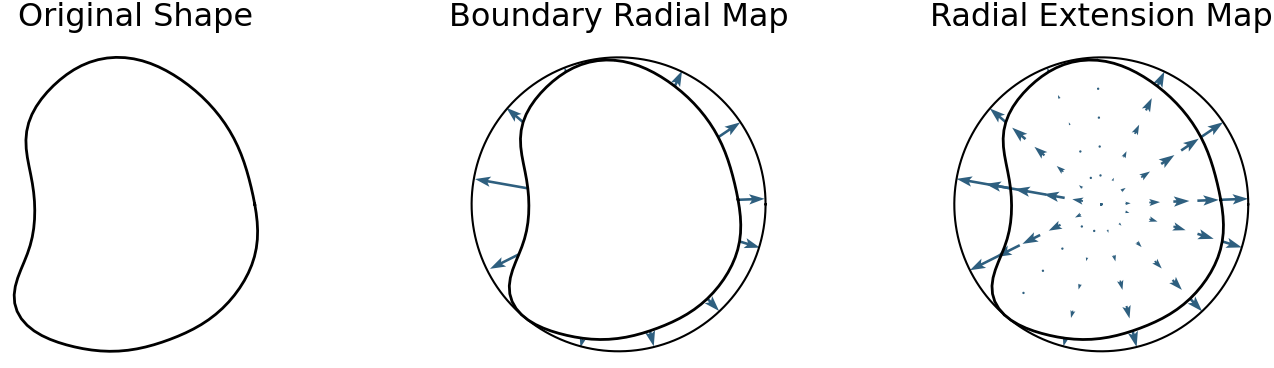}
\caption{Radial mapping of a star-convex shape to the unit disk. Left: original shape. Middle: radial boundary deformation map. Right: radial extension map.}
\label{fig:radial_map}
\end{figure}

\begin{definition}[Radial Extension]\label{def:radially_const}
Let $\Theta_d \subset \mathbb{R}^{d-1}$ denote the angular parameter domain (e.g., $\Theta_2 = [0, 2\pi)$ for $d=2$ and $\Theta_3 = [0,\pi] \times [0,2\pi)$ for $d=3$). Let $\psi_S : \partial B_r \to \partial S$ be a boundary deformation map
expressed in spherical coordinates as
\[
\psi_S(\theta) = (\nu(\theta), \theta), \qquad \forall \theta \in \Theta_d,
\]
for some $\nu \in C^1(\Theta_d; \mathbb{R}_+)$.

We define the \emph{radial extension} $\Psi_r : \overline{B_r} \to \overline{S}$ as
\[
\Psi_r(r,\theta) := (r\,\nu(\theta), \theta), 
\qquad \forall (r,\theta) \in [0,1] \times \Theta_d.
\]
\end{definition}

Figure~\ref{fig:radial_map} illustrates the radial boundary deformation map and its extension applied to a star-convex shape. The radial extension map shows how the original shape is stretched radially onto the reference unit circle while preserving angular coordinates.

The main limitation of the radial extension is that it fails the injectivity assumption in Definition~\ref{def:def_map} for non-star-convex shapes. This motivates the harmonic extension as a less restrictive, but more complex alternative for a broader class of smooth closed shapes:

\begin{figure}[!t]
\centering
\includegraphics[width=3.5 in]{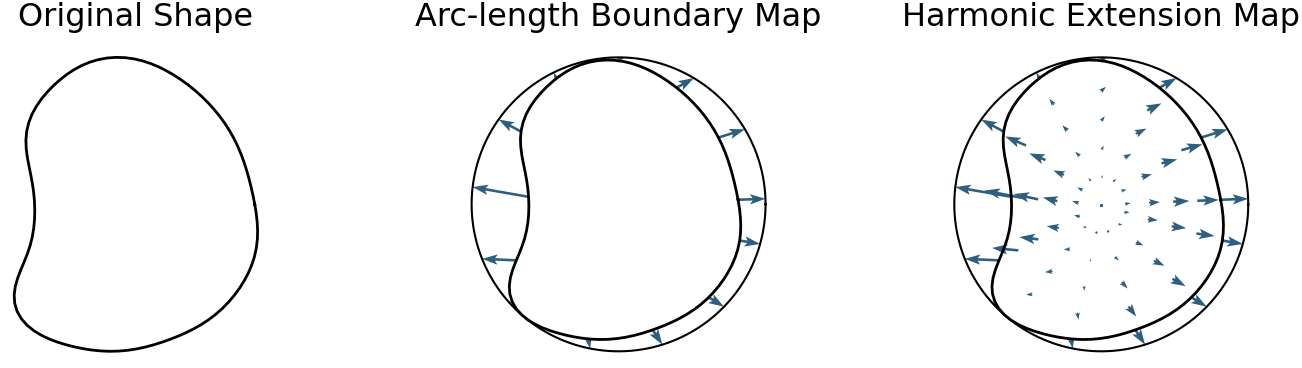}
\caption{Harmonic mapping of a star-convex shape to the unit disk. Left: original shape. Middle: arc-length boundary map. Right: harmonic extension map.}
\label{fig:harmonic_map}
\end{figure}

\begin{definition}[Harmonic Extension]\label{def:harmonic_ext}
Let 
$\psi_S :  \partial B_r \rightarrow \partial S$ 
be a boundary deformation map. The \emph{Harmonic extension} $\Psi_h \in C^2(B_r;\mathbb{R}^d)$ is the function solving
\begin{equation}
    \left\{\begin{array}{ll}
        \Delta \Psi_h(x) = 0, &\forall  x \in B_r\\
        \Psi_h(x) = \psi_S(x),& \forall x\in  \partial B_r,
    \end{array}\right.
\end{equation}
where the Laplacian $\Delta \Psi_h(x)\in \mathbb{R}^d$ is applied component-wise.
\end{definition}
Figure~\ref{fig:harmonic_map} illustrates the harmonic extension on the same star-convex shape as in Fig.~\ref{fig:radial_map} using arc-length parametrization as the boundary deformation map. Unlike the radial extension map which constrains every interior point to lie on the same ray from the origin to the boundary, the harmonic map is free to bend smoothly into the interior, making it well-defined for non-star-convex shapes.

We provide the following result for the radial and harmonic extensions:

\begin{theorem}[Diffeomorphism properties of extensions]\label{thm:diffeomorphism_extensions}
Let $\psi_S : \partial B_r \to \partial S \subset \mathbb{R}^d$ be a $C^1$ boundary deformation map.

\begin{itemize}
\item[(i)] (\emph{Radial extension}) 
Assume that $S\subseteq \Omega$ is star-convex with respect to the origin and that
\[
\psi_S(\theta) = \nu(\theta)\,\omega(\theta), \qquad \nu \in C^1(\partial B_r)\textrm{ with }\nu(\theta) > 0,
\]
where $\omega(\theta) \in  \mathbb{R}^d$ denotes the unit vector pointing in direction $\theta$, e.g, in 2D $\omega(\theta) = (\cos(\theta),\sin(\theta))$.

Then, the radial extension
\[
\Psi_r(r,\theta) = r\,\nu(\theta)\,\omega(\theta)
\]
defines a $C^1$-diffeomorphism from $\overline{B_r} \setminus \{0\}$ onto $\overline{S} \setminus \{0\}$, which extends continuously to the origin.

\item[(ii)] (\emph{Harmonic extension)}  

Let $d=2$ and $\Psi_h \in C^2(B_r;\mathbb{R}^d)\cap C^1(\overline{B_r};\mathbb{R}^d)$ the harmonic extension of $\psi_S$ from Definition~\ref{def:harmonic_ext}.
Then $\Psi_h$ is a $C^1$-diffeomorphism of $\overline{B_r}$ onto $\overline{S}$ if and only if
\begin{equation}\label{eq:non_deg}
\det D\Psi_h > 0 \quad \text{on } \partial B_r.
\end{equation}
\end{itemize}
\end{theorem}
\noindent The proof is given in Appendix \ref{app:theorem}. There is no proof for the diffeomorphism property of the harmonic extension in $d=3$ dimensions. Therefore, we only use the harmonic extension in two dimensions and restrict 3D applications to the star-convex case. 

For shapes on a common scale and centered at the origin\footnote{This can always be achieved w.l.o.g.~by a trivial pre-processing of the data.}, the PF-T is invariant to translation and uniform scaling, as well as to reparametrization of the boundary, since the deformation map $\Psi$ is defined intrinsically through the closest-point projection in Eq.~\eqref{eq:def_bdy} rather than through any specific parametrization of $\partial S$. It is, however, equivariant under rotations and reflections of the shape. To eliminate this remaining ambiguity, we formulate rotation- and reflection-invariant features from spectral representation of the PF-T.

\subsection{Spectral PF-T}\label{subsec:spectral pf-t}

If the reference domain is rotationally symmetric, such as the unit disk, an angular spectral transformation of the PF-T can be used to guarantee invariance to rotations and reflections of the shape descriptors. 
In 2D, this transformation is the Fourier transform, in 3D the spherical harmonics expansion. Early foundational work established Fourier descriptors in $2D$ as an effective tool for analyzing closed planar curves, enabling the characterization of shape invariants and symmetries with robustness to rotation and scaling \cite{fft_plane_closed_curves}. We hence define the spectral PF-T:

\begin{definition}[Spectral Push-Forward Transform]
Let $S \subseteq \Omega \subseteq \mathbb{R}^d$ be a shape and let
$f_S : S \to \mathbb{R}$ be a scalar field defined on $S$. Denote by
$f_{S^*} : \overline{B_r} \to \mathbb{R}$ its PF-T to the unit ball
$B_r \subset \mathbb{R}^d$.

\begin{itemize}
\item \textbf{Fourier representation ($d=2$).}
For a fixed radius $r \in (0,1)$, the restriction
$f_{S^*}(r,\cdot)$ is a function on the circle of radius $r$ centered at the origin ($\partial B_r$ for $r=1$)
with Fourier coefficients
\[
c_k(r)
:=
\frac{1}{2\pi}
\int_0^{2\pi}
f_{S^*}(r,\theta)\,
e^{-ik\theta}\, \mathrm{d}\theta,
\qquad k \in \mathbb{Z}.
\]

Given a truncation order $N_F \in \mathbb{N}$, define
\[
c(r)
=
\bigl(c_k(r)\bigr)_{0\le k < N_F}
\in \mathbb{C}^{N_F}.
\]

The normalized Fourier descriptor at radius $r$ is then
\[
\mathbf{c}^{F}_{S^*}(r)
:=
\frac{
\bigl(|c_k(r)|\bigr)_{0\le k < N_F}
}{
\sum_{k=0}^{N_F-1}|c_k(r)|
}
\in \mathbb{R}^{N_F}.
\]

\item \textbf{Spherical harmonics representation ($d=3$).}

For a fixed radius $r \in (0,1)$, the restriction
$f_{S^*}(r,\cdot)$ is a function on the sphere of radius $r$ centered at the origin ($\partial B_r$ for $r =1$).
Its spherical harmonic expansion is
\[ f_{S^*}(r,\theta,\phi) \approx \sum_{\ell=0}^{L} \sum_{m=-\ell}^{\ell} c_{\ell,m}(r)\, Y_\ell^m(\theta,\phi), \]
where $L$ is the maximum spherical harmonic degree and $c_{\ell,m}(r)$ are the corresponding spherical harmonic coefficients.
For each degree $\ell$, we define the spectral energy \[ E_\ell(r) := \sum_{m=-\ell}^{\ell} |c_{\ell,m}(r)|^2. \]
The spherical harmonic descriptor at radius $r$ is defined as \[ \mathbf{c}^{H}_{S^*}(r) := \bigl(E_\ell(r)\bigr)_{\ell=0}^{L} \in \mathbb{R}^{L+1}. \]

\end{itemize}
The \emph{spectral push-forward transform} $f_{S^*}\to \mathbf{c}_d$, is defined by
\[
\mathbf{c}_d
=
\begin{cases}
\mathbf{c}^{F}_{S^*}, & d=2,\\
\mathbf{c}^{H}_{S^*}, & d=3.
\end{cases}
\]
\end{definition}

The following result establishes the invariance to rotation and reflection of the spectral PF-T.

\begin{proposition}\label{prop:invariance}
Let $S \subseteq \Omega \subseteq \mathbb{R}^d$ be a shape,
$f_S : S \to \mathbb{R}$ a scalar field, and
$f_{S^*} : \overline{B_r} \to \mathbb{R}$ its PF-T representation.

For any rotation or reflection $\mathbb{K}\in O(d)$, define the transformed field
\[
f_{S_{\mathbb{K}}^*}(x)
:=
f_{S^*}(\mathbb{K}x).
\]

Then the spectral transform $f_{S^*}\to \mathbf{c}_d$ is invariant, i.e.,
\[
\mathbf{c}_d(f_{S_{\mathbb{K}}^*})(r)
=
\mathbf{c}_d(f_{S^*})(r),
\qquad
\forall r\in(0,1].
\]

In words:
\begin{itemize}
    \item for $d=2$, the normalized Fourier descriptor
    $\mathbf{c}^{F}_{S^*}(r)$ is invariant under planar rotations and reflections;

    \item for $d=3$, the spherical harmonic descriptor
    $\mathbf{c}^{H}_{S^*}(r)$ is invariant under rotations and reflections.
\end{itemize}
\end{proposition}
\noindent The proof is given in Appendix \ref{app:proposition}.

\subsection{Finite-Dimensional PF-T}

In practice, we compute a finite-dimensional approximation of the PF-T when implementing it algorithmically. We construct this finite-dimensional approximation in a polynomial spectral space to preserve the continuity and differentiability of the PF-T.

Let $\Omega := (-1,1)^d$ denote the $d$-dimensional hypercube. The polynomial spaces $\Pi_{n}(\Omega) \subseteq C^{\infty}(\Omega)$ of $l^\infty$-degree $n \in \mathbb{N}$ are then defined as:

\begin{definition}[Polynomial Space]\label{def:pol}
    Let $\{x^\alpha\}_{\alpha \in A_{n,d}}$ be the canonical basis $x^\alpha:= \prod\limits_{i=1}^d x_i^{\alpha_i}$ for all $\alpha\in A_{n,d}$, where $A_{n,d}$ is the multi-index set $A_{n,d}:=\{\alpha \in \mathbb{N}^d:\|\alpha\|_{\infty}\leq n\}$ with $|A_{n,d}| = (n+1)^d$. A \emph{polynomial space of degree $n\in \mathbb{N}$} is defined as the span of the canonical basis $\Pi_{n}:=\mathrm{span}\{x^\alpha\}_{\alpha\in A_{n,d}}$. We denote by $\Pi_{n}(U):=\{Q|_{U}:Q\in \Pi_{n}\}$ the restriction of the polynomial space $\Pi_n$ to a domain $U\subseteq \mathbb{R}^d$.
\end{definition}

In curved domains, we consider finite-dimensional spectral spaces generated by either Fourier modes ($d=2$) or spherical harmonics ($d=3$):

\begin{definition}[Finite-Dimensional Spectral Space]
For $d=2$, let
\[
A_n
:=
\{k\in\mathbb{Z} : |k|\le n\},
\]
and define the space of trigonometric polynomials of order $n\in\mathbb{N}$ by
\[
\mathcal{T}_n^{2}
:=
\mathrm{span}
\left\{
e^{ik\theta}
\right\}_{k\in A_n}.
\]

For $d=3$, let
\[
\Lambda_n
:=
\left\{
(\ell,m)
\,:\,
0\le \ell \le n,
\;
-\ell\le m\le \ell
\right\},
\]
and define the spherical harmonic space of degree $n$ by
\[
\mathcal{T}_n^{3}
:=
\mathrm{span}
\left\{
Y_\ell^m(\theta,\phi)
\right\}_{(\ell,m)\in\Lambda_n},
\]
where $Y_\ell^m$ denotes the spherical harmonic of degree $\ell$ and order $m$.

We then define the \emph{finite-dimensional spectral space}
\[
\mathcal{T}_n
:=
\begin{cases}
\mathcal{T}_n^{2}, & d=2,\\
\mathcal{T}_n^{3}, & d=3.
\end{cases}
\]

For a domain $U\subseteq \mathbb{R}^d$, we denote by
\[
\mathcal{T}_n(U)
:=
\{f|_U : f\in\mathcal{T}_n\}
\]
its restriction to $U$.
\end{definition}

These continuous objects are computationally represented as polynomial surrogates :

\begin{definition}[Polynomial Surrogates]\label{rem:poly_surrogate}
A \emph{polynomial surrogate} of a function $f:U\to\mathbb{R}^k$ is an approximation $f_\xi\in \Pi_{n},\, \mathcal{T}_n$, obtained by minimizing the discrete residual of the equation satisfied by $f$ over a set of collocation points, typically the Legendre grid~\cite{Neg_Sob_cardona,SuarezCardona_2023}.
\end{definition}

Since $f_\xi \in C^\infty(\Omega)$, its derivatives are computed exactly through polynomial differentiation matrices, yielding a smooth and differentiable approximation.
We further define the finite-dimensional boundary deformation map :

\begin{definition}[Finite-Dimensional Boundary Deformation Map]
Let
$f:S\to\mathbb{R}$
be a scalar field defined over a shape
$S\subseteq \Omega:=(-1,1)^d$.
We define the \emph{finite-dimensional boundary deformation map}
\[
\psi_\xi
:= \mathrm{Id}+\nu_\xi,
\]
where
\(
\nu_\xi\in \mathcal{T}_n(\partial B_r)^d
\)
is the solution of
\[
\inf_{\nu_\xi\in\mathcal{T}_n(\partial B_r)^d}
\frac{1}{N_S}
\sum_{i=1}^{N_S}
\left|
x_i
+
%\nu_\xi(x_i)\eta_r(x_i)
\nu_\xi(x_i)%\eta(x_i)
-
x_S(x_i)
\right|^2.
\]

Here,
\(
\{x_i\}_{i=1}^{N_S}\subseteq \partial B_r
\)
are sampled boundary points,
$x_S(x_i)\in\partial S$ denotes the corresponding target boundary point.
\end{definition}

Using Definition~\ref{def:harmonic_ext} and the non-degeneracy condition from~Eq.~\eqref{eq:non_deg}, we define the finite-dimensional radial and harmonic extensions of $\psi_\xi$:

\begin{definition}[Finite-Dimensional Extension]\label{def:FD_ext}
    Let $\partial S\subseteq \Omega$ be a shape and $\psi_\xi\in \mathcal{T}_n(\partial B_r)^d$ a finite-dimensional boundary deformation map. Then, we define the \emph{finite-dimensional extension} $\Psi_{\xi}$ as either one of:
    \begin{itemize}
        \item \textbf{Finite-Dimensional Radial Extension}  %\Psi_{\xi}\in \mathcal{T}_n(B_r)$ 
        in spherical coordinates, as
        \begin{equation}
            \Psi_{\xi}(r,\theta) := r\psi_\xi(\theta)\, .
        \end{equation}
        
        \item \textbf{Finite-Dimensional Harmonic Extension} $\Psi_{\xi}\in (\Pi_n(B_r))^d$ as the polynomial surrogate solving 
        \begin{equation}
            \inf\limits_{\Psi_{\xi}\in (\Pi_n(B_r))^d} \mathcal{L}^n_\Delta[\xi]+\mathcal{L}^n_b[\xi]+\mathcal{L}^n_{r}[\xi],
        \end{equation}
        with 
        \begin{align}
             &\mathcal{L}^n_\Delta[\xi]:=\sum\limits_{\alpha\in A_{n,d}}\left(\sum_{i=1}^d( \Delta\Psi_{\xi_{x_i}}(P_n))^2_\alpha \right)\omega_\alpha\\&
             \mathcal{L}^n_b[\xi]:= \frac{1}{N_S}\sum_{i=1}^{N_S}\sum_{j=1}^d(\Psi_{\xi_{x_j}}(x_i)-(\psi_{\xi}(x_i))_{x_j})^2,
             \\& \mathcal{L}^n_{r}[\xi] := -\frac{\lambda}{N_S}\sum_{i=1}^{N_S} \log\bigl(\det\bigl(D\Psi_{\xi}(x_i)\bigr)\bigr),
        \end{align}
    \end{itemize}
where $P_n\subseteq \Omega$ is the Legendre grid and $\lambda \in\mathbb{R}_+$ a regularization parameter.
\end{definition}
\noindent The terms $\mathcal{L}^n_\Delta[\xi]$ and $\mathcal{L}^n_b[\xi]$ ensure that the result is a valid harmonic extension according to Def.~\ref{def:harmonic_ext}, and the term $\mathcal{L}^n_{r}[\xi]$ enforces a diffeomorphism according to Theorem~\ref{thm:diffeomorphism_extensions}.
The finite-dimensional PF-T can be applied to any scalar-valued function defined over the domain of a shape. A particularly useful choice for describing the geometry of the shape is the signed distance function (SDF) of that shape. SDFs provide a continuous implicit representation of geometry, encode both boundary and interior structure, and expose singular structures related to the medial axis~\cite{Siddiqi1999}.
\section{Signed Distance Functions (SDF)}\label{sec:signed_distance_function}

We describe in detail the application of the PF-T to SDFs for shape characterization. This provides continuous shape features that are invariant under the Euclidean group. We first recall the mathematical formulation of SDFs, their connection to the Eikonal equation, and their role in encoding topological structures such as the medial axis of a shape.

We start by defining the signed-distance function of a shape $\partial S\subseteq \Omega$ as the distance of any point in $\Omega$ to the closest point on $\partial S$ with the sign distinguishing the inside and outside of the closed shape:

\begin{definition}[Signed-Distance Function]\label{def:SDF}
The \emph{Signed Distance Function (SDF)} of the shape
$\partial S \subseteq \Omega$, denoted $\phi_S : \Omega \to \mathbb{R}$, is
\begin{equation}
\phi_S(x) := \sigma(x)\,\min_{y\in\partial S} \|x - y\|_2,\,\,
\sigma(x) = \begin{cases} +1, & x\in S,\\ -1, & x\in S^c.\end{cases}
\end{equation}
\end{definition}

The SDF $\phi_S$ is a level-set function according to Definition~\ref{def:level_set}. Figure~\ref{fig:sublevelsets_sketch} illustrates the zero, superlevel, and sublevel sets of the signed distance function for a star-convex shape.

\begin{figure}[!t]
\centering
\includegraphics[width=2.5in]{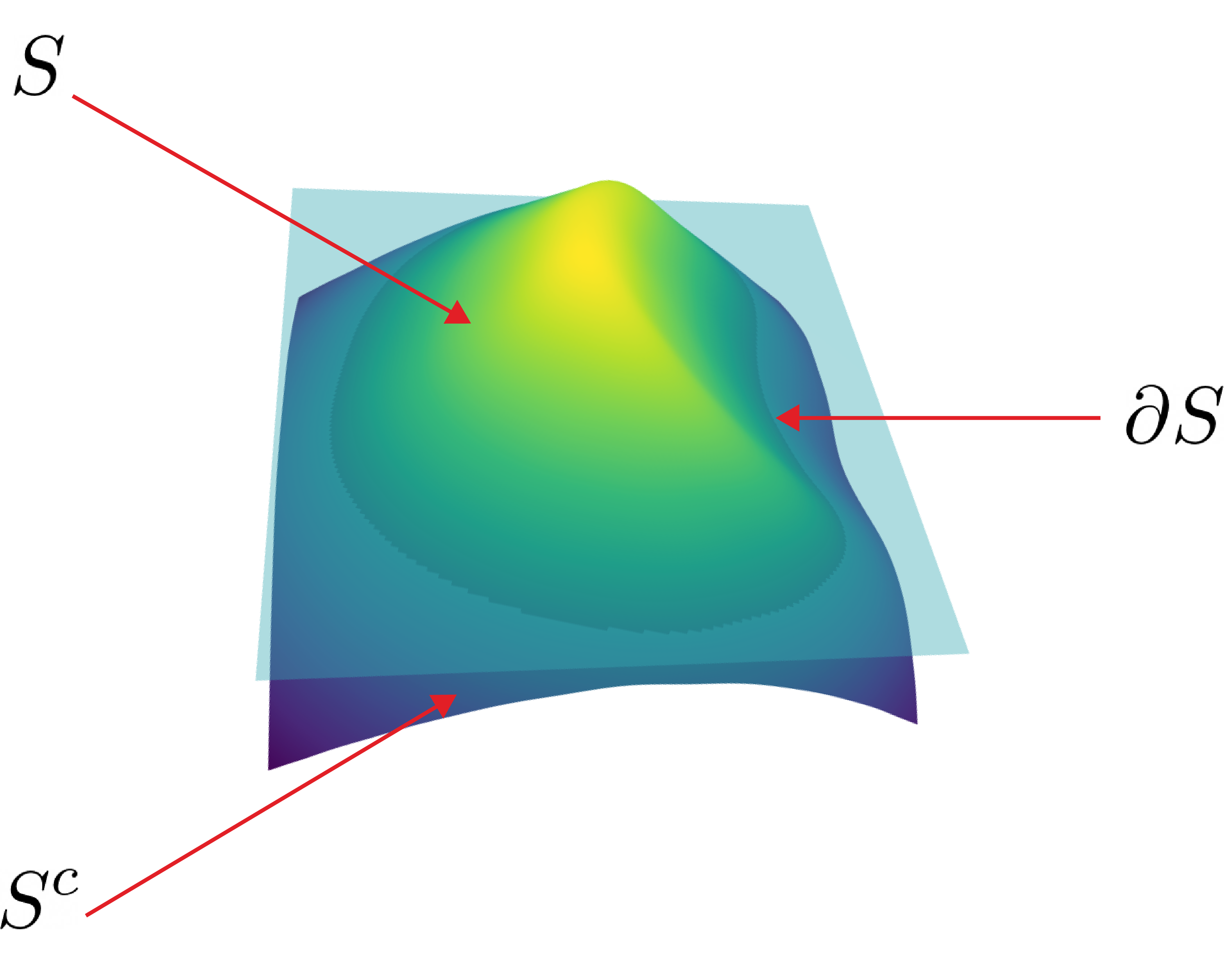}
\caption{Example of the signed distance function of a star-convex shape, showing the zero level-set $\partial S$ indicated by the cutting plane, the superlevel-set $S$ (interior points, positive sign) in lighter colors, and the sublevel-set $S^c$ (exterior points, negative sign) in darker colors.}
\label{fig:sublevelsets_sketch}
\end{figure}

Following Refs.~\cite{Sakai1996-kl,Osher2003-ky}, the SDF $\phi_S$ of a shape $\partial S$ can be determined as the viscosity
solution~\cite{Crandall1983-yo} of the Eikonal equation:

\begin{definition}[Viscous Eikonal Equation]\label{def:eik_loss}
Let $\partial S \subseteq \Omega$ be a given shape. The \emph{viscous Eikonal equation} seeks a function $\phi_S^\mu \in C^2(\Omega)$ satisfying
\begin{equation}
    \left\{
    \begin{aligned}
        &|\nabla \phi_S^\mu|^2 - \mu \Delta \phi_S^\mu = 1 \quad \text{in } \Omega,\\
        &\phi_S^\mu = 0 \quad \text{on } \partial S.
    \end{aligned}
    \right.
\end{equation}

The variational viscous Eikonal problem is to find $\phi \in H^2(\Omega)$ minimizing
\begin{equation}
\begin{aligned}
    \inf_{\phi \in H^2(\Omega)} &
    \int_{\Omega} \left(|\nabla \phi|^2 - \mu \Delta \phi - 1\right)^2 \, \mathrm{d}x\\&
      + \int_{\partial S} (\phi|_{\partial S})^2 \, \mathrm{d}\mathcal{H}(x),
      \label{eq:varEik}
\end{aligned}
\end{equation}
with $\mu\in \mathbb{R}_{+}$ the viscosity parameter. 
\end{definition}
Figure~\ref{fig:sdf_example} shows an example of a binary mask of a two-dimensional shape, the SDF obtained by solving the viscous Eikonal equation, and a 3D visualization of the SDF with zero-level contours.

In practice we solve the variation form in Eq.~\eqref{eq:varEik} to obtain a globally smooth and differentiable approximation of the SDF. We use polynomial surrogates~\cite{Neg_Sob_cardona} and Sobolev Cubatures~\cite{SuarezCardona_2023,cardona2025variationalframeworkalgorithmiccomplexity} to compute the finite-dimensional polynomial SDF:

\begin{definition}[Polynomial Signed-Distance Function]\label{def:polySDF}
    Let $\partial S \subseteq \Omega:=(-1,1)^d$ be a shape. We define the \emph{polynomial SDF} as the polynomial surrogate $\phi_\xi\in \Pi_n(\Omega)$ solving
    \begin{equation}
    \begin{aligned}
  \inf\limits_{\phi_\xi\in \Pi_n(\Omega)}&\sum\limits_{\alpha\in A_{n,d}}\left( \|\mathbb{D}_\nabla \phi_\xi(P_n)\|_2^2-1-\mu\mathbb{D}_\Delta\phi_\xi(P_n)\right)^2\omega_\alpha\\&
    +\frac{1}{N_s}\sum_{i=1}^{N_s}\phi_\xi(x_i)^2,
            \end{aligned}
            \end{equation}  
where $\mathbb{D}_\nabla,\mathbb{D}_\Delta\in \mathbb{R}^{|A_{n,d}|}$ are the polynomial differentiation matrices for the $\nabla$ (Nabla) and $\Delta$ (Laplace) operators, respectively.
\end{definition}

From this, we can define the finite-dimensional PF-SDF:

\begin{definition}[Finite-Dimensional Push-Forward Signed Distance Function (PF-SDF)]
The \emph{finite-dimensional PF-SDF} $\phi_{\xi^*} : B_r \to \mathbb{R}$ is obtained by composing a polynomial SDF $\phi_\xi$ with the harmonic or radial extension $\Psi_\xi$ (see Definition~\ref{def:FD_ext}). More precisely, we define
\begin{equation}
    \phi_{\xi^*}(x) := \phi_{\xi}(\Psi_{\xi}(x)), 
    \qquad \forall x \in B_r,
\end{equation}
where $\Psi_{\xi} : \overline{B_r} \to \overline{S}$ is either the radial Fourier or the harmonic polynomial extension map.
\end{definition}

\subsection{Medial axis of a shape}\label{subsec:medial_axis_structure}

The solution of the inviscid Eikonal equation is Lipschitz continuous but not differentiable. In particular, there exists a set of measure zero where the derivative of the solution $\phi_S$ is discontinuous. This singular set is closely related to the medial axis of the shape $\partial S$ \cite{Xia2009, Chazal2005}, which has been shown to be a powerful geometric descriptor of shape similarity \cite{Torsello2004}.

The viscous Eikonal equation (Def.~\ref{def:eik_loss}), which we solve here, has smooth solutions $\phi_S^\mu \in C^2(\Omega)$ throughout the domain. Nevertheless, for small viscosity $\mu$, the solution retains a diffuse signature of the singular medial set, allowing the medial axis and related geometric features to be approximated. This is illustrated in Fig.~\ref{fig:skeleton_sdf_star}, where the gradients and divergence of the viscous SDF highlight the skeletal structure of the given shape, as also previously demonstrated \cite{Siddiqi1999}. For comparison, we show the \texttt{scikit-image} skeletonization result~\cite{skimage} for the same example shape. 
The information about the medial axis of a shape is contained in its diffuse signature in the polynomial SDF.

\begin{figure}[!t]
\centering
\includegraphics[width=3.5 in]{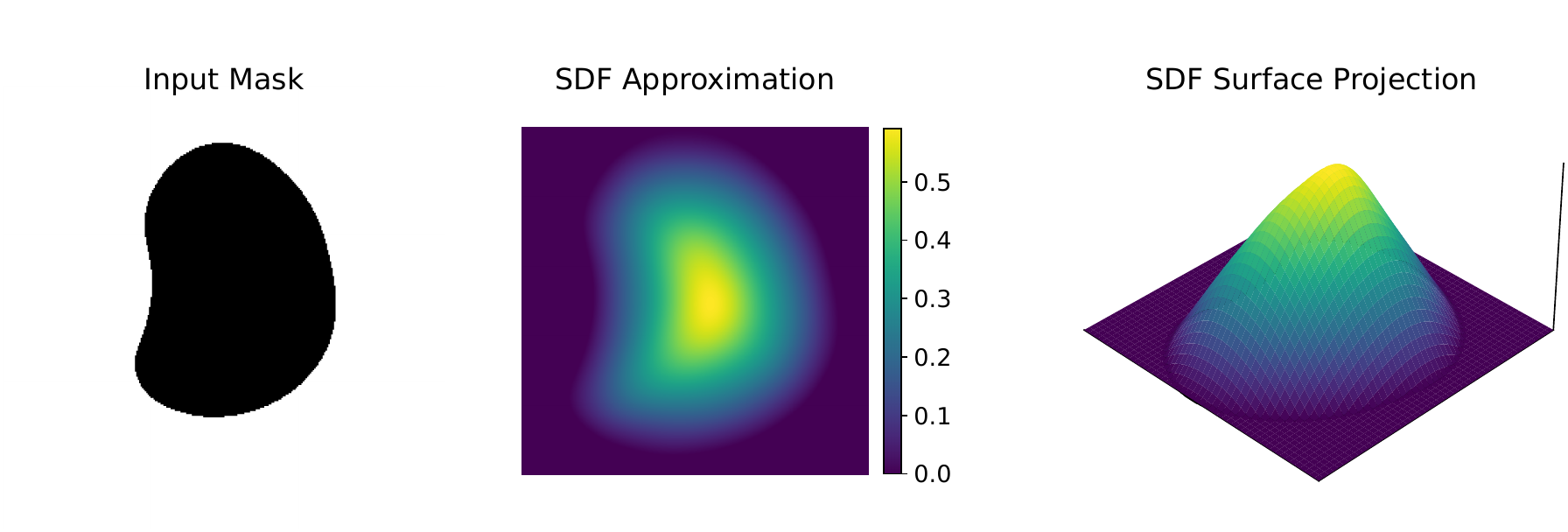}
\caption{From left to right: Input binary mask of a star-convex shape, signed distance function (SDF) approximation obtained by solving the viscous Eikonal equation, and 3D visualization of the SDF with zero-level contour highlighted in black, representing the reconstructed shape boundary.}
\label{fig:sdf_example}
\end{figure}

We start by defining the medial axis of a shape as the set of points inside the shape that are equidistant to at least two points on the shape \cite{Chazal2005}:

\begin{definition}[Medial Axis]
Let $\partial S \subseteq \Omega$ be a shape and $\phi_S : S \to \mathbb{R}$ its SDF. The medial axis (or topological skeleton) is defined as
\begin{align*}
M_S := \{& x \in S : \exists\, x_l, x_r \in \partial S:x_l\neq x_r,
\\& \|x - x_l\|_2 = \|x - x_r\|_2 =\phi_S(x)\}.
\end{align*}
\end{definition}

\begin{lemma}\label{lemma:sing_sdf}
Let $S \subseteq \Omega$ and $\phi_S$ its SDF. Then, $\phi_S$ is not differentiable at any point $c \in M_S$.
\end{lemma}
\noindent The proof is given in Appendix \ref{app:lemma}.

\begin{figure}[!t]
\centering
\includegraphics[width=3.5 in]{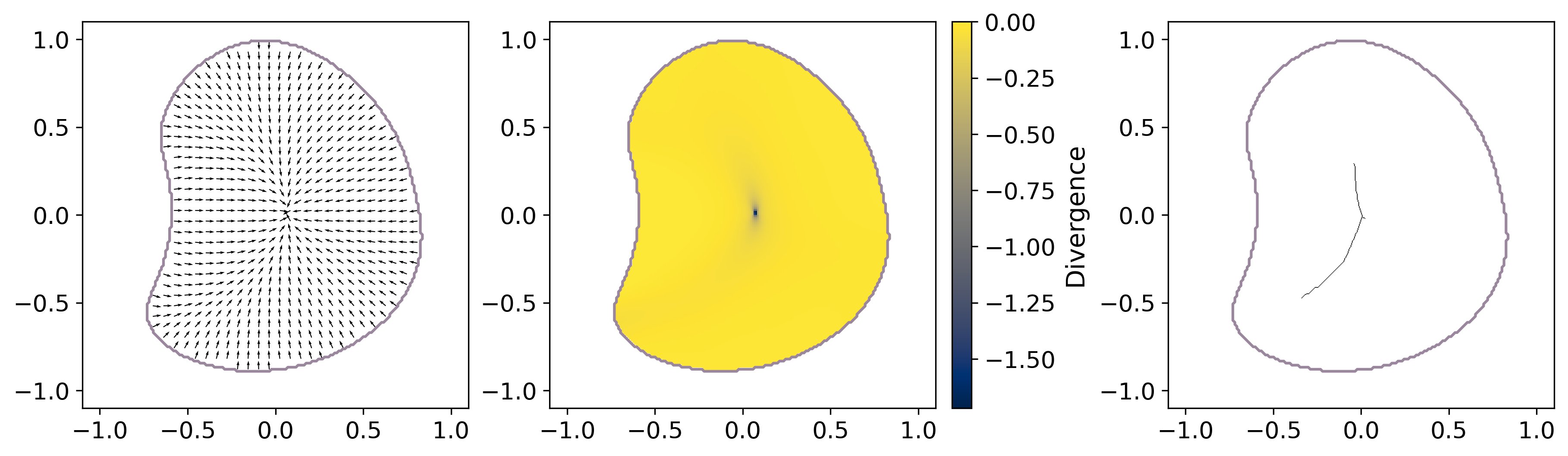}
\includegraphics[width=3.5 in]{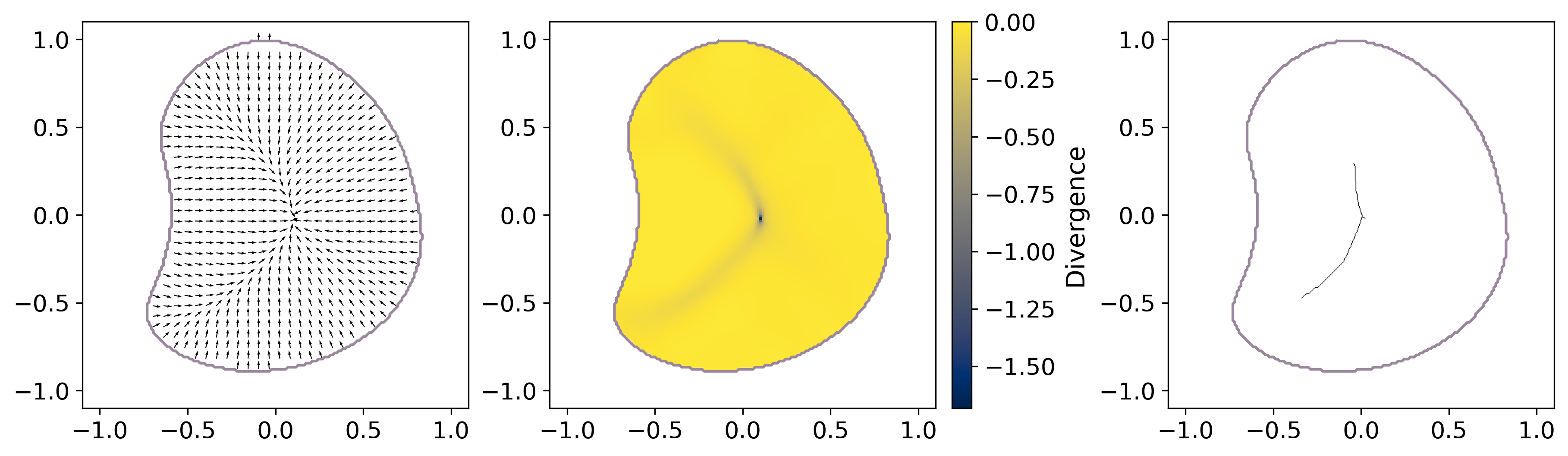}
\caption{Signature of a shape skeleton in the smooth polynomial SDF. Top row: solution from the viscous Eikonal equation with viscosity $0.6$; bottom row: viscosity $0.001$. Left to right in each row: normalized gradients (Eq.~\eqref{eq:normgrad}), divergence of the normalized gradient field, and a reference skeleton obtained using the \texttt{scikit-image} skeletonization algorithm on the original binary shape mask (identical in both rows).}
\label{fig:skeleton_sdf_star}
\end{figure}

The divergence $\nabla \cdot \eta$ of the normalized SDF gradient
\begin{equation}
    \eta(x) = \frac{\nabla \phi_\xi(x)}{\|\nabla \phi_\xi(x)\|}\, ,
    \label{eq:normgrad}
\end{equation}
is low along the medial axis of the shape. In all experiments in Section~\ref{sec:experiments}, we use $\nabla\cdot \eta$ directly as the \emph{skeleton channel} of a PF-SDM. Beyond feature extraction, the same divergence field can be post-processed into a discrete skeleton point cloud and further into a smooth parametric medial-axis curve, providing an intrinsic, landmark-free longitudinal coordinate for elongated shapes. The full extraction procedure is described in Appendix \ref{app:spine}
 and illustrated on a real biological 3D shape.

\section{Experiments and Results}\label{sec:experiments}

Since the polynomial SDF encodes both the boundary geometry of a shape and its interior medial-axis structure, it is a natural choice to combine with the PF-T. The resulting representation can compare shapes while retaining volumetric and skeletal information. We empirically evaluate the performance and robustness (to noise and shape perturbations) of the resulting representation before analyzing its discriminative power in real-world data sets. Finally, we show an application of PF-SDM to a biological problem.

\subsection{Invariance and interpretability}
\label{subsec:invariance}

We first evaluate whether the PF-SDM satisfies the invariance properties required for meaningful shape comparison and whether the resulting descriptors exhibit an interpretable structure in a controlled setting. We consider a synthetic 2D data set comprising four regular polygons (equilateral triangle, square, regular pentagon, and regular hexagon) and an approximately symmetric five-petaled flower~\cite{Gandhi2021}. All shapes are centered at the origin, scaled to be inscribed in the unit circle,
and embedded in the domain $ \Omega := (-1,1)^2 $. For each shape, we generate ten variants covering rotations, reflections, translations, uniform scaling, and additive Gaussian noise (see Appendix \ref{app:synthetic} for details), yielding a total of 50 contours.
Since every shape in this data set is star-convex, we use radial extension (Definition~\ref{def:radially_const}); harmonic extension yields comparable results, but radial extension is computationally more efficient for this shape class.

We compute full pairwise Euclidean distance matrices for all 50 shapes represented using three morphometrics:
\begin{itemize}
    \item \textbf{PF-SDM}: normalized FFT coefficients of the PF-SDF with $ N_F = 15 $ truncation order.
    \item \textbf{EFA}: elliptical Fourier descriptors with the first 15 harmonics, computed using the \href{https://pyefd.readthedocs.io/en/latest/}{\texttt{pyefd}} and normalized to remove rotation, translation, and scale.
    \item \textbf{GPA}: generalized Procrustes analysis applied to equidistant boundary landmarks sampled at equal arc-length intervals from each contour, with pairwise Procrustes distances computed between aligned shapes.
\end{itemize}

\begin{figure}[!t]
\centering
\includegraphics[width=3.5 in]{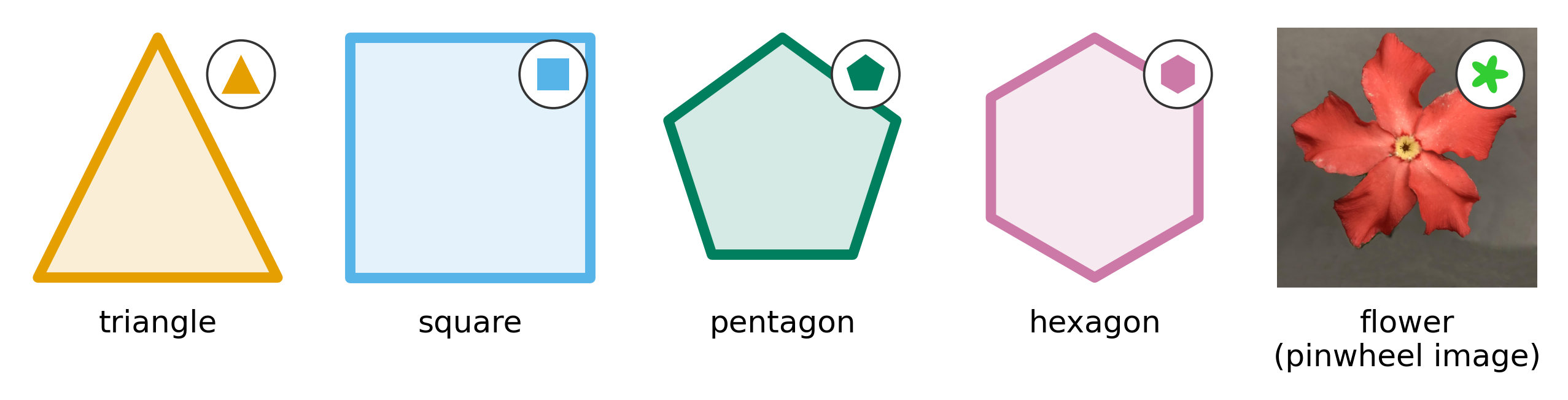}

\includegraphics[width=3.5 in]{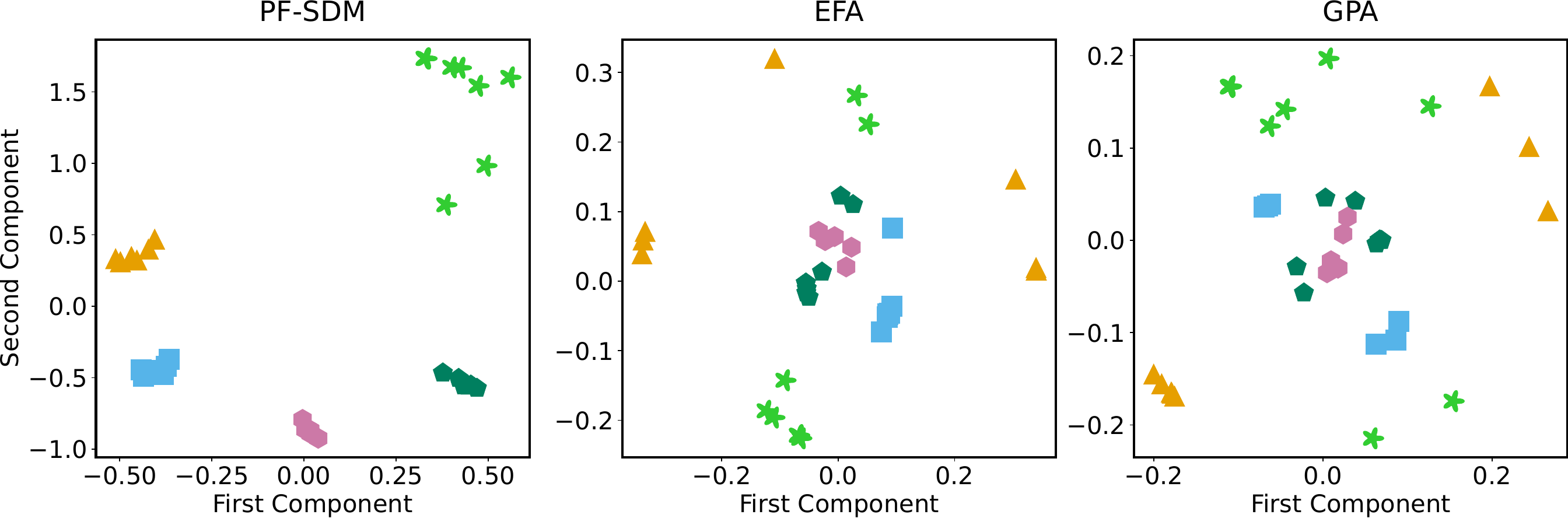}
\caption{Top: 2D synthetic data set of regular polygons and the approximately symmetric five-petaled flower with their legend markers. Bottom: MDS embeddings of Euclidean pairwise distance matrices for PF-SDM (left), EFA (center), and GPA (right) on the synthetic 2D data set. Each of the 50 points represents one of the five shape classes. Colors match the inset legend.}
\label{fig:mds_invariances}
\end{figure}

Figure~\ref{fig:mds_invariances} shows the data classes considered as well as the first two multi-dimensional scaling (MDS) components of the resulting pairwise distance matrices for each method. For PF-SDM, all augmented variants of a given shape collapse to a tight cluster, confirming invariance to the applied shape-preserving transformations and robustness to noise. In contrast, both GPA and EFA exhibit noticeably larger intra-class variance, reflecting their sensitivity to noise and residual alignment errors.
Beyond invariance, the PF-SDM embedding reveals an interpretable organization of shape classes. The irregular flower is clearly separated from the regular polygons, while the polygons themselves are arranged according to their geometric similarity: shapes that more closely approximate a circle (e.g., pentagon and hexagon) cluster closer together, whereas shapes that deviate more from circularity (e.g., triangle) are embedded farther away. The PF-SDM embedding thus reflects meaningful geometric relationships between shapes.

\subsubsection*{Shape hierarchy and rotational symmetry}
While the full PF-SDM FFT spectrum emphasizes global shape similarity, the spectral representation allows explicit control over which geometric properties are emphasized. In particular, removing the Direct Current (DC) component suppresses information related to overall circularity, while using cosine distance emphasizes relative angular variation rather than absolute magnitude. For readability, we denote by $c_k := \bigl[\mathbf{c}^F_{S^*}\bigr]_k$ the $k$-th entry of the normalized Fourier descriptor from Section~\ref{subsec:spectral pf-t}. If $c_0$ is omitted, the induced distances become primarily sensitive to rotational symmetry.

\begin{figure}[!t]
\centering

\includegraphics[width=3.5 in]{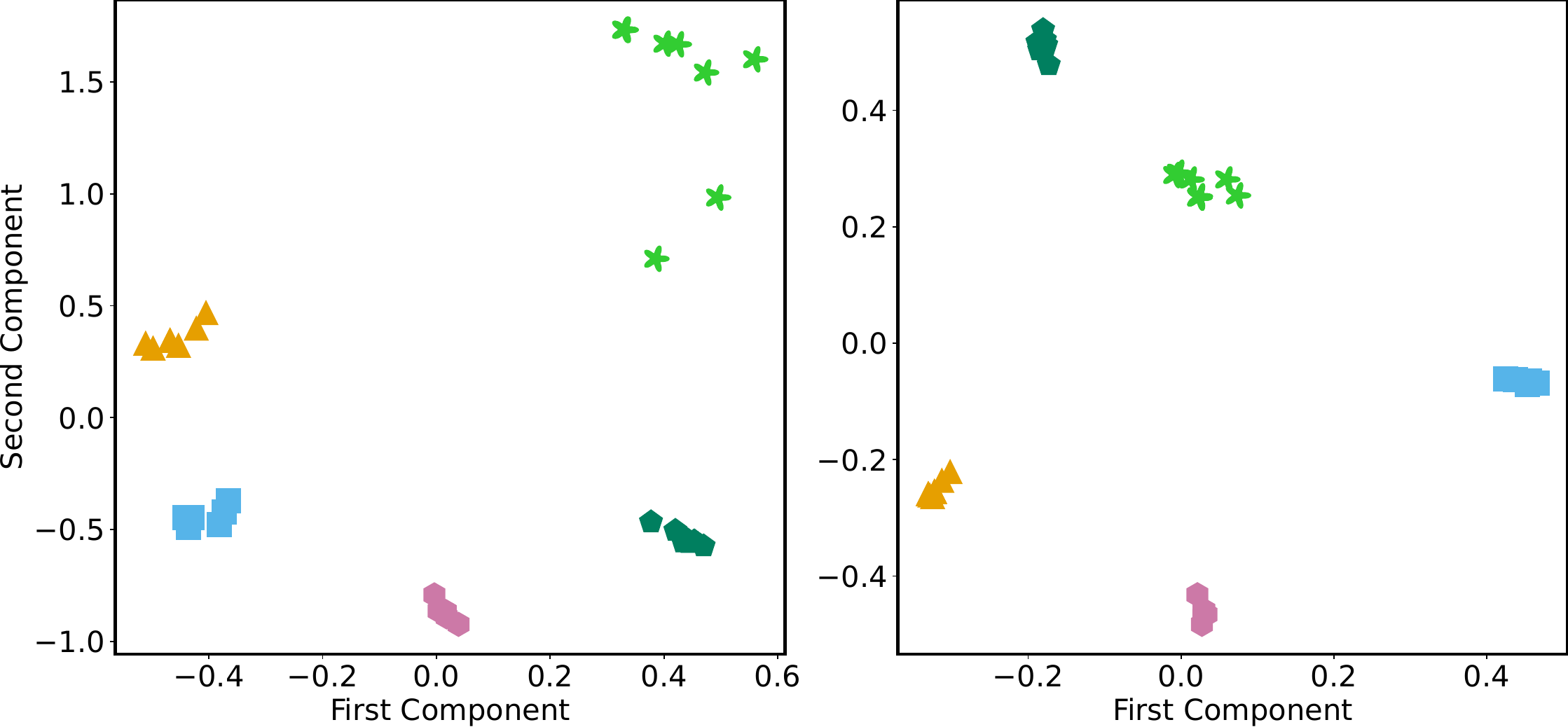}
\caption{MDS embeddings of PF-SDM distance matrices computed using the full normlized Fourier descriptors $\mathbf{c}^F_{S^*}$ with Euclidean distance (left), and using the FFT spectrum excluding $c_0$ with cosine distance (right). Excluding $c_0$ emphasizes rotational symmetry over global shape similarity.}
\label{fig:fft_rotation_symmetry}
\end{figure}

This is illustrated in Fig.~\ref{fig:fft_rotation_symmetry}. When the full FFT spectrum is used with Euclidean distance, the embedding reflects global geometric similarity. In contrast, when $c_0$ is excluded, and cosine distance is used, the embedding reorganizes according to rotational symmetry: the five-petaled flower clusters near the pentagon, reflecting their shared five-fold symmetry. A similar effect is observed for other shapes with related symmetry structure, such as the triangle and hexagon, which exhibit three- and six-fold rotational symmetry, respectively.

\begin{figure}[!t]
\centering

\includegraphics[width=3.5 in]{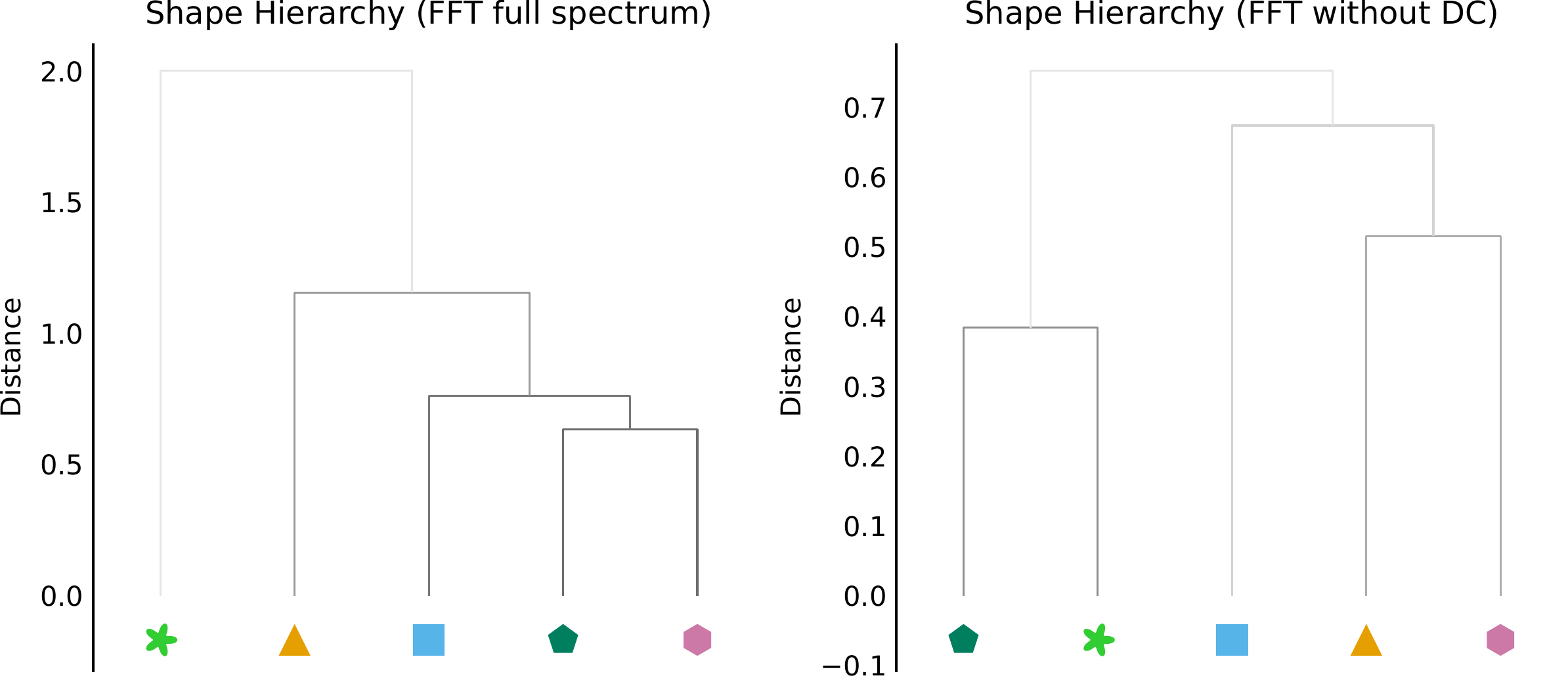}

\caption{Hierarchical clustering of shapes based on PF-SDM FFT spectra. Left: full descriptor $\mathbf{c}^F_{S^*}$, emphasizing shape similarity. Right: excluding $c_0$, emphasizing rotational symmetry groups.}
\label{fig:shape_dendogram}
\end{figure}

This embedding structure becomes even more evident in the hierarchical clustering shown in Fig.~\ref{fig:shape_dendogram}. Using the full descriptor, the dendrogram mainly reflects global geometry. Excluding $c_0$ reorganizes the hierarchy around symmetry order, with shapes sharing a similar rotational structure grouped closer together.

\begin{figure}[!t]
\centering

\includegraphics[width=3.5 in]{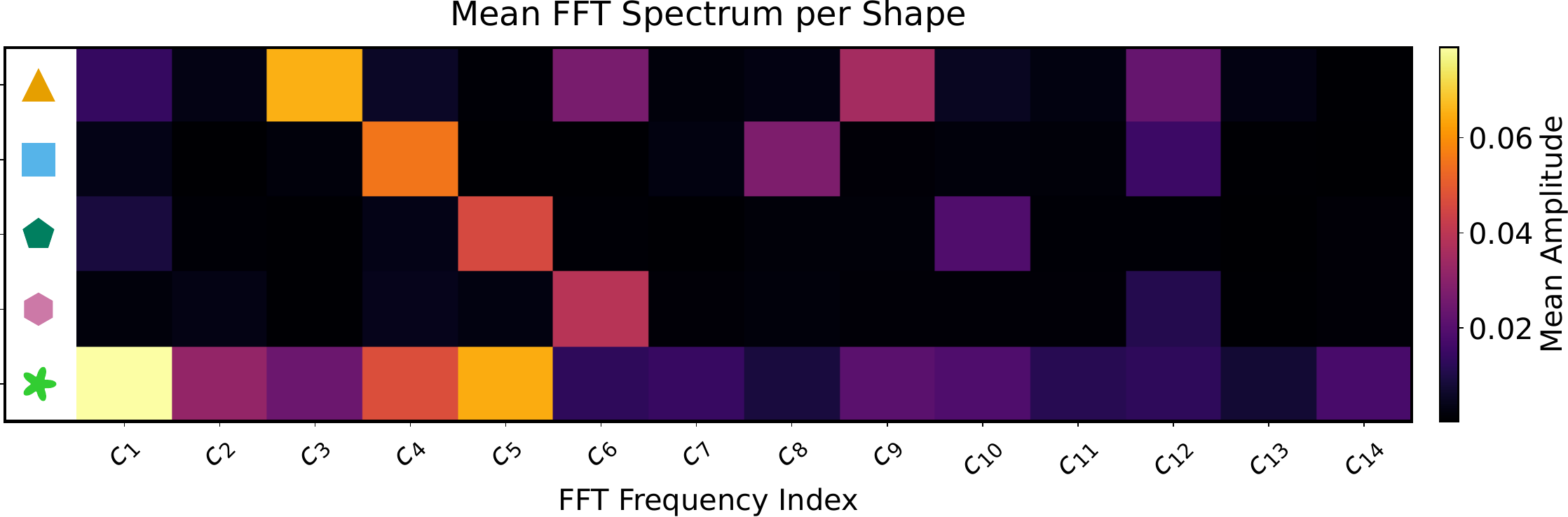}

\caption{Mean normalized Fourier descriptors per shape class, illustrating the relationship between dominant frequency components and rotational symmetry.}
\label{fig:mean_fft_spectrum}
\end{figure}

This reorganization can be explained by the mean normalized Fourier descriptors shown in Fig.~\ref{fig:mean_fft_spectrum}. For regular polygons, the $k$-th order in $c_k$ is dominant for shapes with $k$-fold symmetry (e.g., $c_3$ for the triangle, $c_4$ for the square) with higher harmonics showing up attenuated. The flower exhibits a high $c_5$ component, consistent with its five-petaled geometry. Because it is only approximately five-fold symmetric, the first harmonic $c_1$ also carries substantial energy. Yet, the fifth harmonic remains the first non-trivial local maximum in the spectrum. Similarly, the square exhibits its dominant harmonic at $c_4$, but also carries substantial energy at $c_{12}$ a harmonic it shares with the triangle ($c_3, c_6, c_9, c_{12},\ldots$) and the hexagon ($c_6, c_{12},\ldots$), since $12$ is a common multiple of $3$, $4$, and $6$. Under cosine distance on the DC-removed spectrum, this shared high-order harmonic pulls the square closer to the triangle and hexagon in the symmetry-based embedding.

These properties remain true for shapes in a three-dimensional domain $\Omega$. There, we use the unit ball $B_r \subset \mathbb{R}^3$ as reference domain and compare two complementary morphometrics:
\begin{itemize}
    \item \textbf{PF-SDM}: Spherical harmonic descriptors $\mathbf{c}_{S^*}^H$ of the PF-SDF up to expansion order $L=20$ (see Section~\ref{subsec:spectral pf-t}).
    \item \textbf{Spherical parameterization}: The PF-T itself induces a spherical parameterization of the shape. As in classical spherical-harmonic shape analysis we directly expand the PF-T restricted to the surface in spherical harmonics up to order $L=20$. 
    The resulting degree-wise power spectrum, corresponding to the spectral energy $E_\ell$ defined in Section~\ref{subsec:spectral pf-t}, is used as a shape descriptor.
\end{itemize}
These two morphometrics emphasize different aspects of shape: PF-SDM captures the interior structure in $S$, while the spherical parameterization focuses on the surface $\partial S$ and is possible because all benchmark shapes are topological spheres.

Concretely, we consider six star-convex benchmark shapes in 3D: sphere, ellipsoid, cube, cone, pyramid, and a cylindric capsule shape, each augmented by five shape-preserving transformations (translations, rotations, reflections, and uniform scaling) as described in Appendix \ref{app:synthetic}. As in the 2D case, we use radial extension for the PF-T.

\begin{figure}[!t]
\centering
\includegraphics[width=3.5 in]{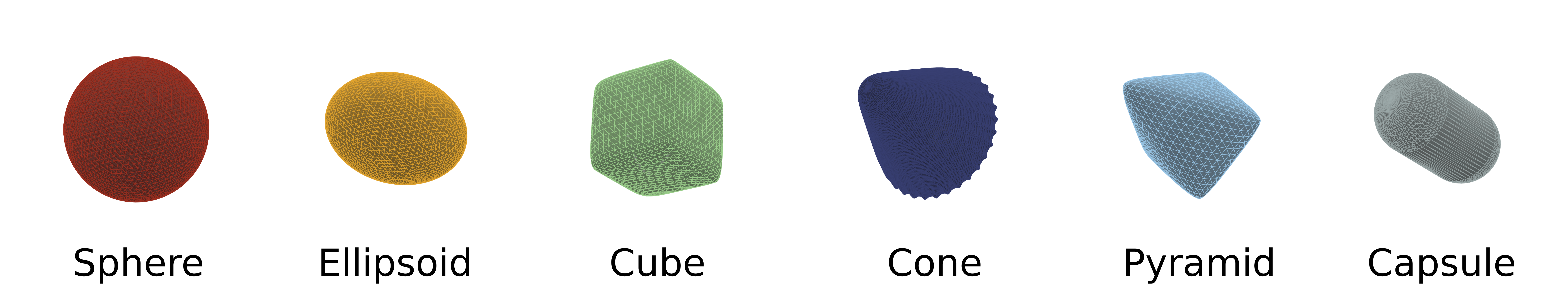}
\includegraphics[width=3.5 in]{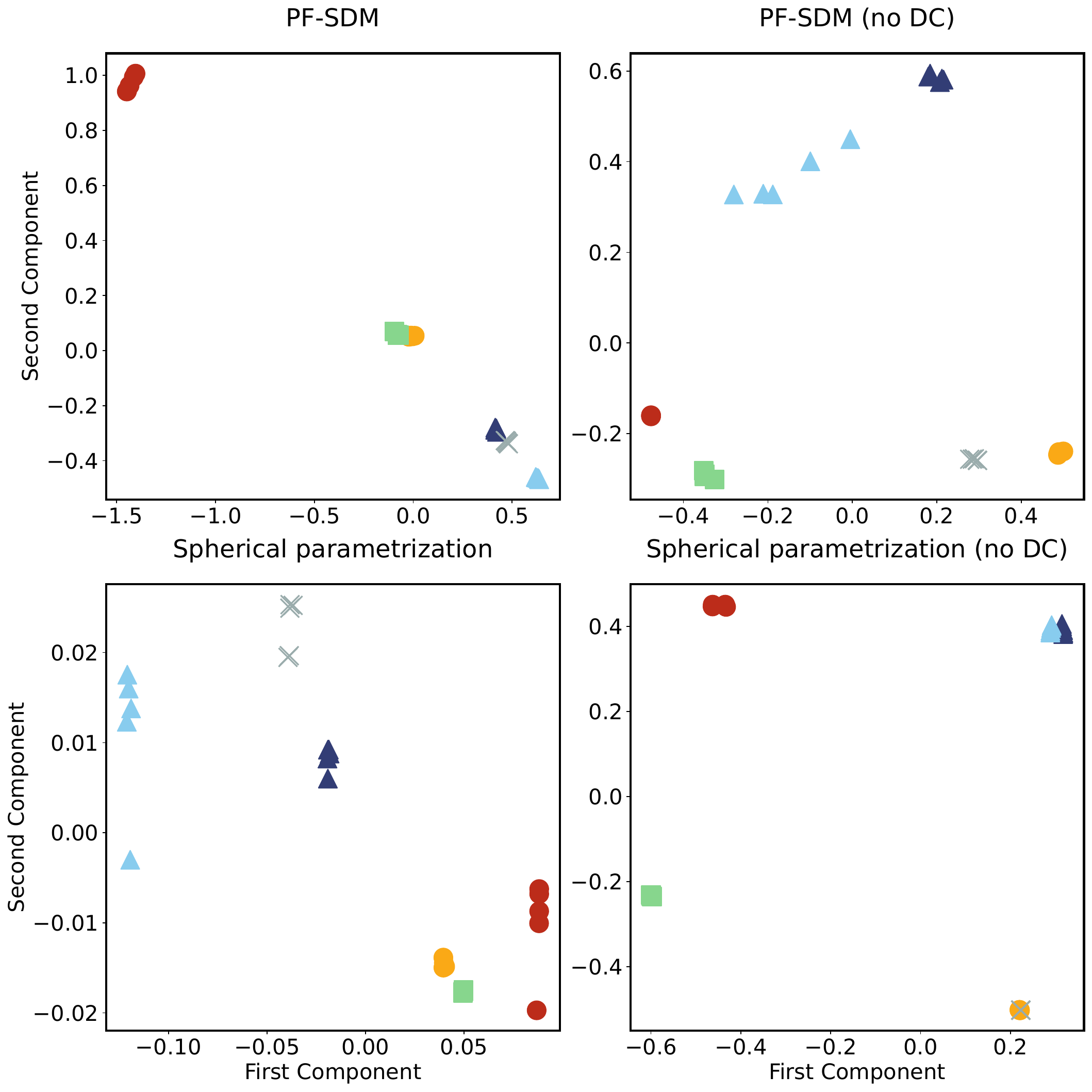}
\caption{Top: the six 3D shape classes. Bottom: MDS embeddings of pairwise distance matrices computed from PF-SDM descriptors (left) and spherical parameterization descriptors (right). Top row: full spectrum with Euclidean distance. Bottom row: spectrum excluding the DC component with cosine distance. Colors match the inset legend.}
\label{fig:3d_mds}
\end{figure}

Figure~\ref{fig:3d_mds} shows the first two MDS components of the resulting pairwise distance matrices. When using the full spectrum with Euclidean distance (left column), both morphometrics organize shapes according to global geometric similarity. In this, PF-SDM is more robust to shape-preserving transformations and noise than direct spherical harmonics expansion.

When the DC component (energy $E_0$) is omitted, and cosine distance is used (right column), the embeddings again reorganize to emphasize rotational symmetry. Then, for both morphometrics, the cone and pyramid are close, as well as the ellipsoid and capsule. 
The spherical parameterization produces a nearly perfect symmetry-based organization: the sphere is isolated, and the cone--pyramid and ellipsoid--capsule pairs collapse almost exactly onto one another. This tighter clustering is expected since the spherical parameterization only evaluates the surface, and all shapes are topological spheres. In contrast, PF-SDM also captures interior structure, which better separates different shapes while still embedding them based on symmetry, as confirmed in Sections~\ref{subsec:benchmarking} and~\ref{subsec:gastruloid}. Overall, these observations confirm the conclusions from the 2D benchmarks.

Together, these results demonstrate that the PF-SDM yields invariant, noise-robust, and interpretable shape descriptors. It also provides explicit control over the geometric features emphasized in shape comparison through simple, principled adaptation of the spectral representation.

\subsection{Discriminative power}\label{subsec:benchmarking}

We assess the discriminative power of the PF-SDM on more complex shapes from real-world computer vision data sets with annotated ground truth. Since these shapes are not generally star-convex, we use harmonic extension of the PF-T in the following experiments.
Specifically, we consider the data sets MPEG-7~\cite{mpeg7_ce_shape1}, a multi-class collection of 1\,400 complex 2D shapes across 70 classes, and BBBC010~\cite{bbbc010} from the Broad Bioimage Benchmark Collection, a binary classification task distinguishing live from dead \textit{C.~elegans} roundworms over 1\,406 individual per-worm masks. Both data sets directly provide binary segmentation masks with examples shown in Fig.~\ref{fig:benchmarking datasets}.

\begin{figure}[!t]
\centering
\includegraphics[width=3.5 in]{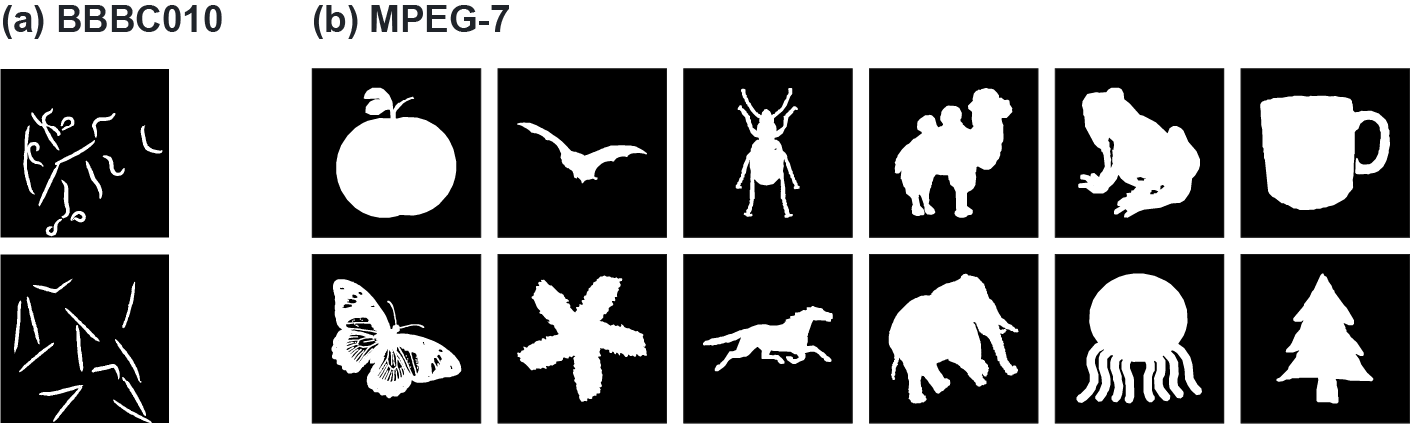}
\caption{Example binary masks from the benchmark data sets. (a) BBBC010 showing per-well masks of microscopy images of roundworms~\cite{bbbc010}: top row shows mostly alive worms and bottom row mostly dead worms. There are multiple worms per image; classification is done on a per-worm basis. (b) Ten example classes from the MPEG-7 shape-classification data set~\cite{mpeg7_ce_shape1}. There is one shape per image.}
\label{fig:benchmarking datasets}
\end{figure}

On these data sets, we compare the classification performance of PF-SDM embeddings with classic morphometric descriptors and with deep-learning shape representations taking into account different levels of information: shape, skeleton, intensity. For this, we assume that the input data, e.g.~an image, contain different \emph{channels}.

A channel is any scalar field $f_S : S \to \mathbb{R}$ defined over the  domain of a shape $\partial S$. The PF-T (Definition~\ref{def:def_map}) maps each channel to the reference domain $B_r$, and the spectral PF-T (Section~\ref{subsec:spectral pf-t}) produces a rotation- and reflection-invariant feature vector. In the experiments below, we consider three channels:
\begin{itemize}
    \item \textbf{Shape channel} ($\mathrm{SDF}$): the polynomial SDF $\phi_\xi$, encoding boundary location and signed distance to the shape $\partial S$.
    \item \textbf{Skeleton channel} ($\mathrm{Skel}$): the divergence of the normalized SDF gradient,
    $\nabla\cdot \eta$ from Eq.~\eqref{eq:normgrad}, which smoothly encodes the medial-axis structure of the shape $\partial S$ (see Section~\ref{subsec:medial_axis_structure}).
    \item \textbf{Intensity channel} ($\mathrm{Int}$): any additional scalar signal $I_S : S \to \mathbb{R}$ defined over the shape domain, such as a pixel intensity field.
\end{itemize}

We refer to descriptors obtained from individual or combined channels as:
\begin{itemize}
    \item \textbf{PF-SDM}: shape channel only.
    \item \textbf{PF-SDM\,+\,Skel}: shape and skeleton channels.
    \item \textbf{PF-SDM\,+\,Int}: shape and intensity channels.
    \item \textbf{PF-SDM\,+\,Skel\,+\,Int}: all three channels.
\end{itemize}

For each descriptor variant, we train \emph{independent} logistic-regression classifiers on the spectral features of each active channel, obtaining a per-channel class-probability vector $P_c(y\mid x)$ for
$c \in \mathcal{C} \subseteq \{\mathrm{SDF},\mathrm{Skel},\mathrm{Int}\}$.
The fused prediction is the convex combination
\begin{equation}\label{eq:late_fusion}
    P_{\mathrm{fused}}(y\mid x) := \sum_{c\in\mathcal{C}} \alpha_c\,P_c(y\mid x)\, , \quad \sum_{c\in\mathcal{C}} \alpha_c = 1\, ,\quad \alpha_c\ge 0,
\end{equation}
with weights $\alpha_c$ selected by grid search in the validation split to maximize the target metric (F1-score in Section~\ref{subsec:benchmarking}, balanced accuracy in Section~\ref{subsec:gastruloid}). Ties are resolved by selecting the most uniform weighting, ensuring a deterministic and permutation-invariant fusion. For single-channel descriptors (here, \textbf{PF-SDM}),
Eq.~\eqref{eq:late_fusion} reduces to the underlying classifier.

We adopt a late fusion scheme rather than feature concatenation because it (i) decouples the regularization of channels with different amplitude scales and dimensionalities, (ii) yields interpretable channel-wise contribution weights $\alpha_c$, and (iii) makes single-channel ablations directly comparable across descriptor variants.

All models are evaluated under an identical 70/15/15 train/validation/test split. Deep-learning methods use the training and validation sets for model fitting and regularization; for PF-SDM and the other learning-free descriptors, features are extracted from the combined training and validation sets. In all cases, the features are used to train a logistic regression classifier, and all methods are evaluated on the held-out test set. 

We compare two PF-SDM descriptor variants (defined above) with classification based on 19 geometric region properties from \texttt{OpenCV}~\cite{opencv_library} (listed in Appendix \ref{app:regionprops}),
Elliptical Fourier Descriptors (EFA)~\cite{Kuhl1982} as well as four deep-learning approaches: masked autoencoders (MAE)~\cite{MAE} in three vision transformer (ViT) configurations, both pretrained on ImageNet-1K and trained from scratch on the target masks, SimCLR~\cite{simclr}, O2VAE~\cite{Burgess2024}, and ShapeEmbed~\cite{Annafoix} with and without an appended size feature. Implementation details for all baselines are given in Appendix \ref{app:baselines}.
Classification performance (F1-score) is reported in Table~\ref{table:benchmarkingF1}, and end-to-end runtimes are compared in Table~\ref{table:runtime}.

\begin{table}[ht]
\centering
\caption{Shape classification performance (F1-score, higher is better) on two computer-vision benchmark data sets: MPEG-7 and BBBC010. The most accurate model is highlighted in bold.}\label{table:benchmarkingF1}
\begin{tabular}{|l|c|c|}
\hline
\textbf{Method} & \textbf{MPEG-7} & \textbf{BBBC010} \\
\hline
Region properties         & 0.45 & 0.83 \\
EFA                       & 0.56 & 0.52 \\ 
MAE (ViT-B, pretrained)   & 0.39 & 0.62 \\
MAE (ViT-L, pretrained)   & 0.40  & 0.66 \\
MAE (ViT-H, pretrained)   & 0.37 & 0.69 \\
MAE (ViT-B)               & 0.87 & 0.64 \\
MAE (ViT-L)               & 0.63 & 0.35 \\
MAE (ViT-H)               & 0.56 & 0.35 \\
SimCLR                    & 0.81 & 0.81 \\
O2VAE                     & 0.31 & 0.68 \\
ShapeEmbed                & 0.90  & 0.80  \\
ShapeEmbed + Size         & 0.77 & 0.80  \\
PF-SDM                    & 0.85 & \textbf{0.85} \\
\textbf{PF-SDM + Skel}    & \textbf{0.91} & \textbf{0.85} \\ 
\hline
\end{tabular}
\end{table}

\begin{table}[ht]
\centering
\caption{End-to-end runtimes for the different models on the benchmarks from Table~\ref{table:benchmarkingF1}; h = hours, m = minutes, s = seconds. The fastest model is highlighted in bold.}
\label{table:runtime}
\begin{tabular}{|l|c|c|}
\hline
\textbf{Method} & \textbf{MPEG-7} & \textbf{BBBC010} \\
\hline
Region properties         & 1m 5s   & 5.5s \\
\textbf{EFA}              & \textbf{4s}      & \textbf{3.5s} \\
MAE (ViT-B, pretrained)   & 4m      & 4m \\
MAE (ViT-L, pretrained)   & 10m     & 11m 30s \\
MAE (ViT-H, pretrained)   & 20m 20s & 21m \\
MAE (ViT-B)               & 10h 30m & 10h 26m \\
MAE (ViT-L)               & 17h 45m & 17h 17m \\
MAE (ViT-H)               & 39h 41m & 37h 6m \\
SimCLR                    & 6h 30m  & 6h 33m \\
O2VAE                     & 7h      & 4h 30m \\
ShapeEmbed                & 3h 30m  & 3h 30m \\
ShapeEmbed + Size           & 3h 30m  & 3h 30m \\
PF-SDM                    & 6m36s   & 3m17s \\
PF-SDM + Skel             & 10m49s  & 5m22s\\
\hline
\end{tabular}
\end{table}

Runtimes were measured on an Apple M1 Max (macOS 15.6.1). Models ran on CPU only, using their default settings, which may have included internal Pytorch optimization. We report end-to-end runtimes, which for PF-SDM include image pre-processing, solving the viscous Eikonal equation, computing the divergence for the skeleton information, computing the PF-T, and extracting the shape features. We also include the time for training, grid parameter search on the validation set, and inference of the logistic regression classifier. The higher runtimes on the MPEG-7 data set are due to the higher geometric complexity of the shapes, which leads to longer Eikonal solve times.

The two data sets highlight two different regimes: MPEG-7 offers a multi-class set of shapes with high diversity and rich medial-axis structure (e.g., tools, animals), whereas BBBC010 is a binary classification of \textit{C.~elegans} roundworms all sharing the same topology and mainly differing in boundary curvature.
On MPEG-7, \textbf{PF-SDM\,+\,Skel} achieves the highest F1-score (0.91) of all tested methods, closely followed by ShapeEmbed (0.90) and the from-scratch MAE ViT-B (0.87). Classic approaches like region properties and EFA clearly underperform. The pretrained MAE models also perform worse than their from-scratch counterparts.

On the simpler and more homogeneous shapes of BBBC010, simple region properties outperform all deep-learning baselines with large from-scratch models particularly affected. PF-SDM still performs best among all tested methods, but because all worms share the same skeletal topology, the skeleton channel adds no additional discriminative signal.

Looking at runtimes, both PF-SDM variants are orders of magnitude faster than from-scratch deep-learning baselines and about as fast as forward inference using a pretrained model. Simple geometric features like EFA and region properties are expectedly the fastest. 

These results show that PF-SDM is competitive with state-of-the-art traditional and deep-learning methods for shape analysis both in terms of accuracy and runtime. They also show that adding the skeleton channel can further improve performance on topologically diverse shape sets.

\subsection{Spatiotemporal shape and patterning analysis}\label{subsec:gastruloid}

Having established the performance and invariance properties of PF-SDM, we apply it to a biological problem: characterizing the developmental dynamics of mouse gastruloids.

Mouse gastruloids are three-dimensional \emph{in vitro} models of early embryonic development that self-organize from aggregates of embryonic stem cells and recapitulate key features of body-axis formation~\cite{Savill2025}. A non-negligible fraction of gastruloids, however, fails to elongate along a single axis and instead develops multiple protrusions. The \texttt{Brachyury::mCherry} reporter molecule marks the developing posterior (i.e., tail-end) pole, providing a ground-truth patterning readout. We consider the problem of predicting whether a given gastruloid will develop a single axis or multiple axes before the posterior marker becomes clearly visible. Solving this task requires combining morphological information (shape and skeleton) with molecular patterning (Brachyury intensity) evolving over time.

\begin{figure}[!t]
\centering
\includegraphics[width=3.5 in]{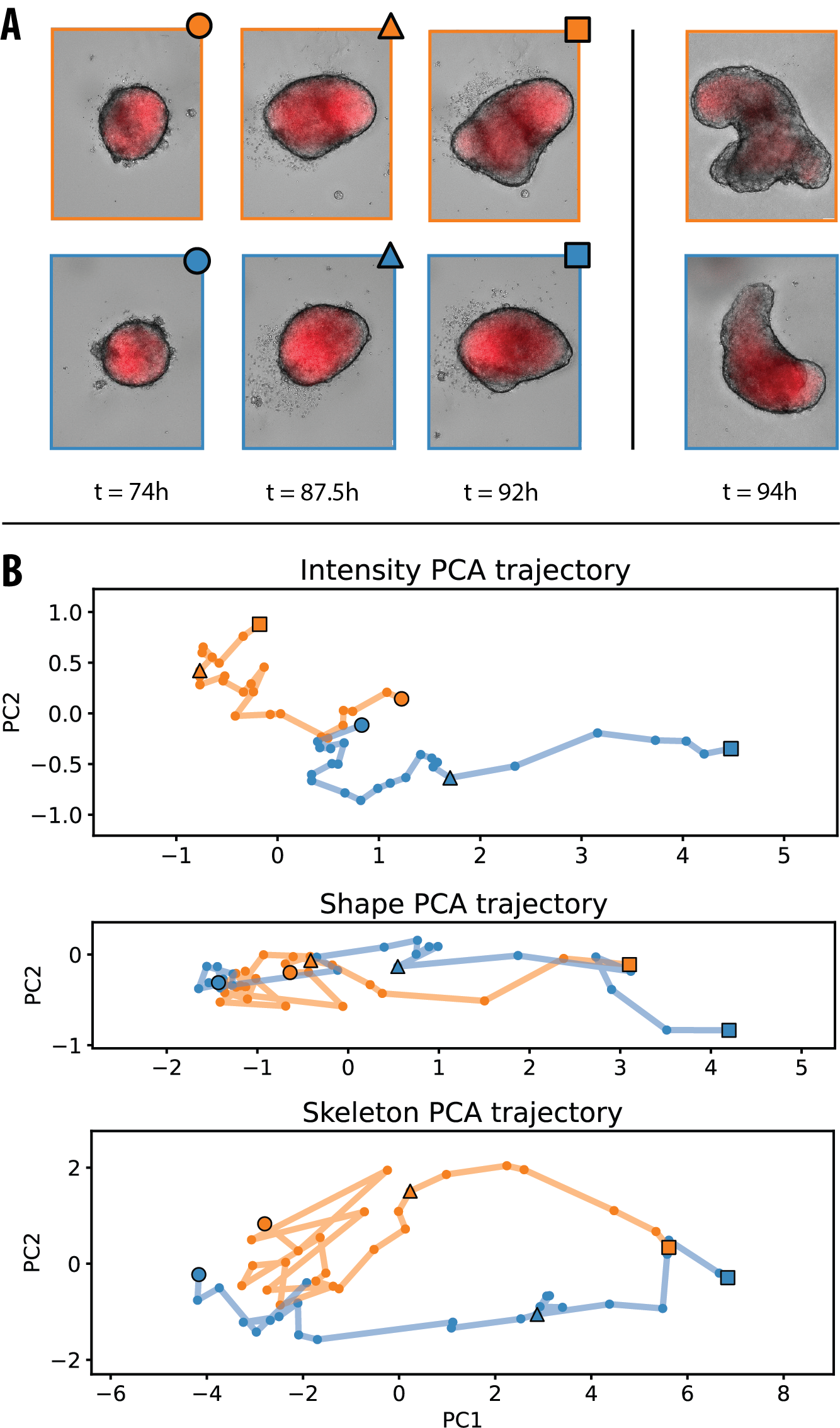}
\caption{Shape analysis of developing mouse gastruloids. (A) Representative gastruloid images from the single-axis (blue frame) and multi-axes (orange frame) classes at early (circle), intermediate (triangle), and late (square) times, along with a posterior reference at 94 h. Brachyury::mCherry intensity is shown in red. (B) PCA trajectories of $N_F=25$ PF-SDM Fourier coefficients c for Brachyury intensity (top), Skeleton information (middle) and shape (bottom) for the examples from (A) across 25 time points.}

\label{fig:gastruloids_pannel}
\end{figure}

Figure~\ref{fig:gastruloids_pannel}(A) shows example microscopy images of  gastruloids from the two classes at three developmental stages (74h, 87.5h, 92h) together with the {\it a-posteriori} reference at 94h. Panel (B) projects the PF-T FFT coefficients of each channel onto their first two principal components, yielding a low-dimensional trajectory for each gastruloid over the 25 imaging time points. The three channels carry visibly complementary information: the intensity trajectory captures the spatial reorganization of the Brachyury reporter as the posterior pole emerges, the skeleton trajectory reflects the evolution of the underlying medial-axis structure, and the shape trajectory tracks the elongation and protrusion patterns of the tissue boundary. The two phenotypes separate earlier in the intensity channel than in the shape channel, where they only diverge once morphological differences become visually apparent. This qualitative observation motivates the joint use of all three channels.

We therefore evaluate the combined \textbf{PF-SDM\,+\,Skel\,+\,Int} descriptor variant from Section~\ref{subsec:benchmarking} with the intensity channel given by the \texttt{Brachyury::mCherry} fluorescence signal on a data set of 78 gastruloids.

We follow the same evaluation protocol as in our previous work~\cite{pfsdm} to enable direct comparison with the two-channel \textbf{PF-SDM\,+\,Int} baseline reported there. For each gastruloid, the spectral PF-T coefficients of every active channel are concatenated across the 25 time points to yield a single sample-level descriptor. We train independent logistic-regression classifiers with balanced class weights for the shape, skeleton, and intensity channels, obtaining per-channel class probabilities $P_{\mathrm{SDF}}, P_{\mathrm{Skel}}, P_{\mathrm{Int}}$ that are combined by the late fusion rule of Eq.~\eqref{eq:late_fusion}. Performance is estimated under repeated stratified shuffle splits ($k=5$, test size $20\%$), with the fusion weights $\alpha_c$ tuned per split to maximize balanced accuracy. We report mean $\pm$ standard deviation across splits. 

The fused model in its present formulation achieves an accuracy of $0.90\pm0.050$ and a balanced accuracy of $0.91\pm0.080$, improving over the previously reported two-channel baseline (accuracy $0.875\pm0.039$, balanced accuracy $0.667\pm0.000$)~\cite{pfsdm}. The shape channel receives the highest weight ($\alpha_{\mathrm{SDF}}=0.41$), while the second highest weight comes from the skeleton channel with ($\alpha_{\mathrm{Skel}}=0.36$), indicating that medial-axis information provides a strong discriminative signal for this task. The slightly increased variance in the balanced accuracy reflects the limited number of minority-class samples.

Together with the results from subsection~\ref{subsec:benchmarking}, this application shows that PF-SDM provides strong discriminative power across shape classes of varying complexity and extends beyond pure geometry to include additional intensity signals through interpretable per-channel fusion. Discriminative power, as measured by the F1-score on the two benchmarking datasets, and by the accuracy and balanced accuracy on the gastruloid dataset, is on par with or better than state-of-the-art deep-learning models while remaining fully deterministic, training-free, and interpretable.

\section{Conclusion}\label{sec:conclusion}

We introduced the \emph{Push-Forward Transform} (PF-T), a continuous, smooth, and deterministic map for comparing scalar functions defined over different shape domains by mapping them to a common reference domain through a structure-preserving diffeomorphism. The framework yields descriptors that are invariant to translation, rotation, reflection, re-parametrization and uniform scaling by construction, while preserving the intrinsic information encoded in the original functions. Applied to signed distance functions (SDF), the PF-T induces the PF-SDF and an associated morphometric, the PF-SDM, in either Fourier (in 2D) or spherical-harmonic (in 3D) bases and provides explicit control over which geometric properties dominate the comparison.

Exploiting the viscous formulation of the Eikonal equation, we derived an algorithm that extracts a parameterized medial axis (spine) directly from the smooth polynomial SDF approximation, providing an intrinsic, landmark-free coordinate for elongated shapes. Across 2D, 3D, and time-resolved data sets, PF-SDM matches or exceeds the classification accuracy of state-of-the-art traditional and deep-learning methods while requiring no training data and running orders of magnitude faster than deep-learning models. In dynamic biological gastruloid data, fusing PF-SDM shape descriptors with PF-T intensity and skeleton features further improved predictive accuracy over the previous state of the art, illustrating the value of the framework for joint geometric and signal analysis.

The proposed framework also has limitations. The PF-T, as was defined here, requires a smooth diffeomorphism between the reference and target domains. It is therefore restricted to shapes that are topologically equivalent to the reference domain. Here, we only considered topological spheres. Extending the framework to other topologies would require choosing a topology-matched reference domain and defining a boundary deformation map and extension that remain diffeomorphic. If the reference domain lacks rotational symmetry, rotation invariance is no longer obtained from the spectral representation and may require an additional transformation.

Furthermore, the parameterized spine algorithm relies on PCA-based ordering of the skeletal point cloud, which restricts its applicability to elongated shapes with a single dominant axis. Extending it to branched and near-isotropic shapes would require replacing it with graph-based traversal.

Finally, the presented PF-T is restricted to scalar-valued functions. Generalizing it to vector- and tensor-valued fields would require extending the notion to the push-forward of differential forms using exterior calculus~\cite{Cartan2006-hr}. 

Several possibilities for future applications of the PF-T follow naturally. First, the entire PF-T computational pipeline---from solving the viscous Eikonal equation to extending the boundary deformation map---is differentiable in its parameters, offering the possibility of end-to-end gradient-based optimization. Second, the framework provides a principled deterministic feature extractor that can be integrated with deep-learning models to combine geometric inductive biases with data-driven flexibility. More broadly, because the PF-T applies to any continuous scalar function defined over a shape domain, the framework provides a unified mathematical formulation for the joint analysis of morphology and spatial signals in any setting where shapes and shape-correlated distributions must be compared invariantly and at scale.

\section{Software and Data Availability}
The 2D implementation of PF-SDM is publicly available at \url{https://git.mpi-cbg.de/mosaic/software/machine-learning/pf-sdm}, together with a \texttt{napari} plugin at \url{https://git.mpi-cbg.de/mosaic/software/machine-learning/pf-sdm-napari-plugin}, as released with our previous work~\cite{pfsdm}. In this work, we additionally release the 3D version of PF-SDM, including parameterized medial-axis (spline) extraction and the synthetic demo shapes and mesh data used in Section~\ref{sec:experiments} and in Appendix~\ref{app:spine}, at \url{https://git.mpi-cbg.de/mosaic/software/machine-learning/3d}.

\section*{Acknowledgments}
\label{sec:acknowledgments}
This work was supported by the German Ministry of Research, Technology and Space (BMFTR, Bundesministerium f\"{u}r Forschung, Technologie und Raumfahrt) as part of the Center for Scalable Data Analytics and Artificial Intelligence (ScaDS.AI) Dresden/Leipzig.
Large language models were used for proof-reading and improving the quality of the text, as well as for code generation in the publicly available implementation.

\bibliographystyle{unsrtnat}

\appendix
\renewcommand{\thefigure}{A\arabic{figure}}
\renewcommand{\thetable}{A\arabic{table}}
\renewcommand{\theequation}{A\arabic{equation}}
\renewcommand{\thealgorithm}{A\arabic{algorithm}}
\setcounter{figure}{0}
\setcounter{table}{0}
\setcounter{equation}{0}

\section{Proof of Theorem~\ref{thm:diffeomorphism_extensions}}\label{app:theorem}
\begin{proof}
Let $x=(x_1,\ldots,x_d)\in\mathbb{R}^d$ be the Cartesian coordinate vector and $\theta\in\Theta_d\subset\mathbb{R}^{d-1}$ the collection of all angular coordinates.
We start by proving (i). By construction, the radial extension
$\Psi_r : \overline{B_r} \setminus \{0\} \to \overline{S} \setminus \{0\}$ is of class $C^1$. Denoting by $D_x\Psi_r$ its Jacobian in Cartesian coordinates and by $D\Psi_r$ its Jacobian in spherical coordinates $(r,\theta)$. The change of variables $(r,\theta)$ to $x$ with Jacobian factor $r^{d-1}J_\omega(\theta)$ gives
\begin{align*}
 \det D_x\Psi_r(x)=\frac{\det D\Psi_r(r,\theta)}{r^{d-1}J_\omega(\theta)} &=  \frac{\nu(\theta)^dr^{d-1}J_\omega(\theta)}{r^{d-1}J_\omega(\theta)} \\&=\nu(\theta)^d>0 ,
\end{align*}
where
$$
J_\omega(\theta):= \det(\omega(\theta),\partial_{\theta_1}\omega(\theta),\dots, \partial_{\theta_d-1}\omega(\theta))\, .
$$
depends only on the angular parametrization.

Thus, $\det D_x\Psi_r(r,\theta) > 0$ for all $r>0$.
For instance, in dimension $d=2$, a direct computation yields
\[
\det D\Psi_r(r,\theta) = r\,\nu(\theta)^2(\cos^2(\theta)+\sin^2(\theta))= r\,\nu(\theta)^2,
\]
and since the deformation map $\nu(\theta)\neq 0$ by definition, $\det D_x\Psi_r(r,\theta)=\nu(\theta)^2>0$.
Injectivity follows from the star-convexity of $\partial S$, which ensures that each point $x \in S \setminus \{0\}$ admits a unique representation of the form
\[
x = \rho\,\omega(\theta(x)), \textrm{ for some } \rho \in (0,\nu(\theta(x))).
\]
This yields a well-defined inverse of $\Psi_r$, proving global injectivity.

Since $\det D\Psi_r > 0$ on $B_r \setminus \{0\}$, the inverse function theorem implies that $\Psi_r$ is locally invertible. Combined with global injectivity, this yields a global $C^1$ inverse on $\overline{S} \setminus \{0\}$. Hence, $\Psi_r$ is a $C^1$-diffeomorphism.

For $d=2$, the proof for the harmonic extension follows by applying the Alessandrini--Nesi theorem~\cite{Giovanni2009}, using that the boundary deformation map satisfies the injectivity and non-degeneracy conditions, i.e., $\det(D\Psi_h)>0$ on $\partial B_r$.
\end{proof}

\section{Proof of Proposition~\ref{prop:invariance}}\label{app:proposition}
\begin{proof}
We prove the claim of the invariance proposition from the main text separately for $d=2$ and $d=3$.

\medskip
\noindent
\textbf{Case $\boldsymbol{d=2}$.}
Any orthogonal transformation $\mathbb{K}\in O(2)$ acts on the angular
variable as
\[
\theta \to \sigma\theta+\alpha,
\qquad
\sigma\in\{1,-1\},
\]
where $\sigma=1$ corresponds to a rotation and $\sigma=-1$ to a
reflection. Hence
\[
f_{S_{\mathbb{K}}^*}(r,\theta)
=
f_{S^*}(r,\sigma\theta+\alpha).
\]
The corresponding Fourier coefficients satisfy
\[
c_k^{\mathbb{K}}(r)
=
e^{i\beta_k}c_{\sigma k}(r)
\]
for some phase factor $e^{i\beta_k}$ with $|e^{i\beta_k}|=1$.
Therefore, $|c_k^{\mathbb{K}}(r)|=|c_{\sigma k}(r)|$.
If $\sigma=1$, this directly yields $|c_k^{\mathbb{K}}(r)|=|c_k(r)|$.
If $\sigma=-1$, then, since $f_{S^*}$ is real-valued,
$c_{-k}(r)=\overline{c_k(r)}$, and hence
$|c_k^{\mathbb{K}}(r)|=|c_{-k}(r)|=|c_k(r)|$.
Thus, in both cases, $|c_k^{\mathbb{K}}(r)|=|c_k(r)|$.
Moreover, the normalization factor is preserved:
\[
\sum_{k=0}^{N_F-1}|c_k^{\mathbb{K}}(r)|
=
\sum_{k=0}^{N_F-1}|c_k(r)|.
\]
Consequently,
$\mathbf{c}^{F}_{S_{\mathbb{K}}^*}(r) = \mathbf{c}^{F}_{S^*}(r)$,
which proves invariance under rotations and reflections in dimension~$2$.

\medskip
\noindent
\textbf{Case $\boldsymbol{d=3}$.}
Let $\mathbb{K}\in O(3)$ be an orthogonal transformation that includes
rotations and reflections. For each degree $\ell$, the spherical harmonic
coefficients of the transformed field are obtained from the original
coefficients by an orthogonal, unitary change of basis within the same degree:
\[
c_{\ell,m}^{\mathbb{K}}(r)
=
\sum_{m'=-\ell}^{\ell}
D^{\ell}_{m,m'}(\mathbb{K})
\,c_{\ell,m'}(r),
\]
where $D^\ell(\mathbb{K})$ denotes the representation matrix associated with
degree $\ell$.
Since $D^\ell(\mathbb{K})$ is unitary, the degree-wise spectral energy is
preserved:
\begin{align*}
E_\ell^{\mathbb{K}}(r)
&= \sum_{m=-\ell}^{\ell} |c_{\ell,m}^{\mathbb{K}}(r)|^2
 = \sum_{m=-\ell}^{\ell} |c_{\ell,m}(r)|^2 \\
&= E_\ell(r),
\qquad \ell = 0,\ldots,L\, .
\end{align*}
Therefore, the spherical harmonic descriptor satisfies
$\mathbf{c}^{H}_{S_{\mathbb{K}}^*}(r) = \mathbf{c}^{H}_{S^*}(r)$
and is invariant under rotations and reflections in dimension~$3$.
\end{proof}

\section{Proof of Lemma~\ref{lemma:sing_sdf}}\label{app:lemma}
\begin{proof}
Let $c \in M_S$, and let $x_l, x_r \in \partial S$ be two distinct closest points, i.e.
\[
\|c - x_l\|_2 = \|c - x_r\|_2 = \phi_S(c) =: d\, .
\]

Define
\[
z_l^\delta := c - \delta \frac{c - x_l}{\|c - x_l\|_2},
\qquad
z_r^\delta := c - \delta \frac{c - x_r}{\|c - x_r\|_2},
\]
for $0 < \delta < d$, and let $x_l^\delta, x_r^\delta \in \partial S$ be the closest points to $z_l^\delta, z_r^\delta$, respectively. Then, we have that $z_l^\delta\to c$ and $z_r^\delta\to c$ as $\delta \to 0$. Moreover, since $\partial S$ is a compact set, there exists a subsequence such that $x_l^\delta\to \overline{x}_l$ and $x_r^\delta \to\overline{x}_r$, for $\overline{x}_l,\overline{x}_r\in \partial S$.

We prove that $\overline{x}_l=x_l$. By optimality of $x_l^\delta$ we have
\begin{align}
    \|z_l^\delta-x_l^\delta\|_2 \leq \|z_l^\delta-x_l\|_2\, .
\end{align}
Taking the limit $\delta\to 0$ and using the continuity of the norm, we obtain
\begin{align}
    \|c-\overline{x}_l\|_2 \leq \|c-x_l\|_2\, .
\end{align}
Since $x_l$ is a closest point to $c$, we also have
\begin{align}
    \|c-x_l\|_2 \leq \|c-\overline{x}_l\|_2\, ,
\end{align}
and therefore $\|c-\overline{x}_l\|_2 = \|c-x_l\|_2$.

To show that the minimizer of $\|c-x\|_2$ is unique, we prove that
\begin{equation}
    \|z_l^\delta -x_l\|_2<\|z_l^\delta-x\|_2,\ \forall x\in \partial S\textrm{, s.t. }\|c-x\|_2 = \|c-x_l\|_2.
\end{equation}
To show this, let $x\in \partial S$ such that $\|c-x_l\|_2^2=\|c-x\|_2^2$, then compute
\begin{align}
    \|z_l^\delta -x_l\|_2^2 = \|c-x_l\|_2^2-2\delta \|c-x_l\|_2+\delta^2.
\end{align}
On the other side,
\begin{align}
    \|z_l^\delta-x\|_2^2= \|c-x\|_2^2-2\frac{\delta}{||c-x_l||_2} \langle c-x,c-x_l\rangle_2+\delta^2.
\end{align}
Using that $\|c-x_l\|_2^2 = \|c-x\|_2^2$ and canceling the common terms, we have that
\begin{align*}
    \|z_l^\delta -x_l\|_2<\|z_l^\delta-x\|_2
    \iff \|c-x_l\|_2^2 > \langle c-x,c-x_l\rangle_2.
\end{align*}
By the Cauchy--Schwarz inequality, $\langle c-x,c-x_l\rangle_2 \leq \|c-x\|_2\|c-x_l\|_2 = \|c-x_l\|_2^2$, and since $x\neq x_l$, the inequality is strict.
This shows that $x_l\in \partial S$ is the unique minimizer of the constrained optimization
\begin{equation}
    \inf\limits_{\stackrel{x\in \partial S}{\|c-x\|_2=\|c-x_l\|_2}}\|z_l^\delta-x\|_2\, .
\end{equation}
Therefore, $\overline{x}_l=x_l$. The proof follows analogously for $x_r^\delta\in \partial S$.
Thus,
\[
\nabla \phi_S(z_l^\delta) = \frac{z_l^\delta - x_l}{\|z_l^\delta - x_l\|_2},
\qquad
\nabla \phi_S(z_r^\delta) = \frac{z_r^\delta - x_r}{\|z_r^\delta - x_r\|_2}.
\]

In the limit $\delta \to 0$, we obtain
\[
\nabla \phi_S(z_l^\delta) \to \frac{c - x_l}{\|c - x_l\|_2},
\qquad
\nabla \phi_S(z_r^\delta) \to \frac{c - x_r}{\|c - x_r\|_2}.
\]

Since $x_l \neq x_r$, these limits differ, and therefore $\phi_S$ is not differentiable at $c$.
\end{proof}

\section{Parametric Medial-Axis Extraction}\label{app:spine}

Given the polynomial SDF $\phi_\xi$ and its normalized-gradient divergence $\nabla\cdot \eta$ (Eq.~\eqref{eq:normgrad} of the main text), Algorithm~\ref{alg:spine} extracts a smooth parametric medial axis $\gamma:[0,1]\to\mathbb{R}^d$ for elongated shapes with a single dominant principal axis.

\begin{algorithm}[H]
\caption{Parametric medial-axis extraction}\label{alg:spine}
\begin{algorithmic}[1]
\REQUIRE Polynomial SDF $\phi_\xi\in \Pi_n(S)$ over shape domain $S$, uniform grid $\mathcal{G}_S\subseteq S$, quantile level $\alpha\in(0,1)$, step size $h>0$, B-spline order $K$
\ENSURE Parametric medial axis $\gamma:[0,1]\to\mathbb{R}^d$

\STATE Compute $\eta(x) = \nabla\phi_\xi(x)/\|\nabla\phi_\xi(x)\|$ and $(\nabla\cdot \eta)(x)$ for all $x\in\mathcal{G}_S$
\STATE Compute the $\alpha$-quantile $q_\alpha$ of $\nabla\cdot \eta$ over $\mathcal{G}_S$
\STATE Form the skeleton point cloud
\[
\mathcal{P}_{\mathrm{skel}} = \{x \in \mathcal{G}_S : (\nabla\cdot \eta)(x) \le q_\alpha\}
\]
\STATE\label{alg:pca} Compute the centroid $\bar{x}$ and first principal component $v_1$ of $\mathcal{P}_{\mathrm{skel}}$; order points by scalar projection $t_i=(x_i-\bar{x})^\top v_1$ and smooth to obtain a discrete centerline $(\hat{x}_1,\ldots,\hat{x}_m)$
\STATE Compute endpoint tangents $\tau_1=\tfrac{\hat{x}_2-\hat{x}_1}{\|\hat{x}_2-\hat{x}_1\|}$ and $\tau_m=\tfrac{\hat{x}_m-\hat{x}_{m-1}}{\|\hat{x}_m-\hat{x}_{m-1}\|}$
\WHILE{$\phi_\xi(\hat{x}_1)>0$}
    \STATE $\hat{x}_1 \gets \hat{x}_1 - h\,\tau_1$
\ENDWHILE
\WHILE{$\phi_\xi(\hat{x}_m)>0$}
    \STATE $\hat{x}_m \gets \hat{x}_m + h\,\tau_m$
\ENDWHILE
\STATE Assign chord-length parameters $s_i\in[0,1]$ to the extended centerline and fit a B-spline of order $K$ by solving
\[
    \min_{P_1,\ldots,P_K}\sum_{i=1}^{m}\Bigl\|\hat{x}_i - \sum_{j=1}^K B_j(s_i)\,P_j\Bigr\|_2^2
\]
\RETURN $\gamma(s)=\sum_{j=1}^K B_j(s)\,P_j$
\end{algorithmic}
\end{algorithm}

\begin{figure}[!t]
\centering
\includegraphics[width=3.5in]{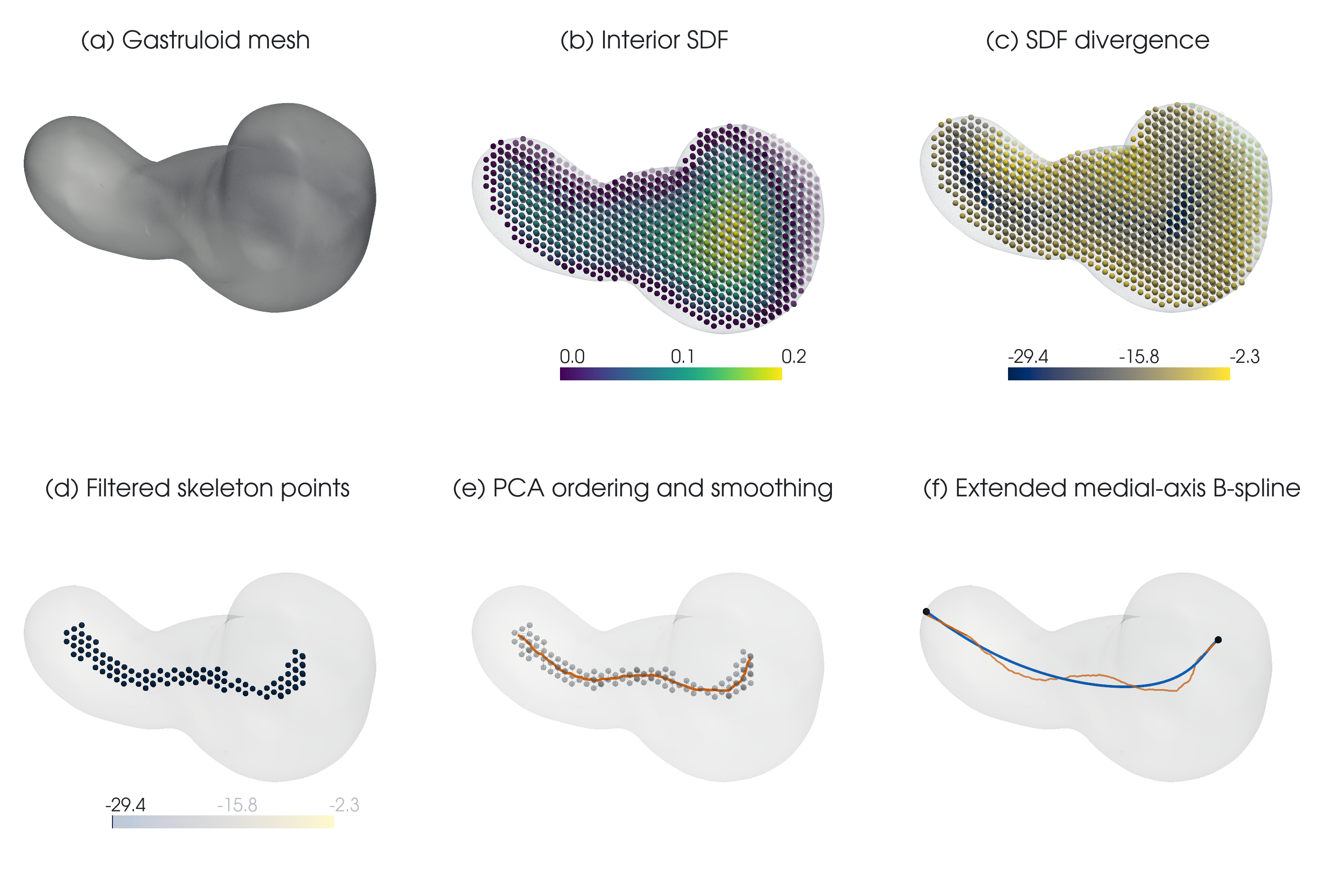}
\caption{Extraction of a parameterized medial axis for a 3D biological gastruloid shape from Ref.~\cite{Savill2025}, following Algorithm~\ref{alg:spine}: (a) shape surface mesh; (b) central plane cut of the polynomial SDF values sampled inside the volume; (c) central plane cut of the divergence of the normalized gradient; (d) filtered skeleton point cloud; (e) ordered and smoothed skeleton points obtained by PCA projection (red); (f) final medial axis as a smooth parametric B-spline (blue).}
\label{fig:spline}
\end{figure}

\section{Synthetic Shape Generation}\label{app:synthetic}

The five 2D shape classes are equilateral triangle, square, regular pentagon, regular hexagon, and the five-petaled flower from~\cite{Gandhi2021}, each class is represented by ten variants:
the original shape, rotations by \(30^\circ,45^\circ,60^\circ\),
reflection about the \(y\)-axis, scalings by \(2\) and \(3\),
translation by \((1,3)\), and Gaussian perturbations with
\(\sigma=0.005\) and \(\sigma=0.01\). For each augmented contour we rasterize a binary image, segment it and follow the same preprocessing pipeline for all. Specifically we extract 100 contour points at equal arc-length intervals, the contours are then centered to the origin and rescaled to be inscribed in the unit circle.

In our synthetic 3D benchmark, each class is represented by five transformed \texttt{trimesh} meshes (\url{https://github.com/mikedh/trimesh}). The base classes are a sphere, ellipsoid, cube, cone, pyramid, and capsule. For each class, we generate five variants using rotations by \(60^\circ,90^\circ,120^\circ,150^\circ,\) and \(180^\circ\) about the axis \((0.5,2,1)\), translations by \((1,-1,2),(2,-2,4),(3,-3,6),(4,-4,8),\) and \((5,-5,10)\), and uniform scalings by \(6,11,16,21,\) and \(26\), respectively. After transformation, each mesh is centered at the origin and uniformly rescaled to fit inside the unit ball. Since these synthetic objects are available directly as surface meshes, we evaluate signed-distance samples from the mesh surface and fit the same polynomial SDF surrogate used by the PF-SDM pipeline. Thus, the subsequent push-forward map, polynomial representation, and spherical-harmonic PF-SDM descriptors remain unchanged.

\section{Region Properties Used as Baseline Features}\label{app:regionprops}

For the region-properties classification baseline, we extract 19 shape features
from binary masks using \texttt{OpenCV}~\cite{opencv_library}:
area, convex area, perimeter, major and minor axis lengths, extent, eccentricity,
solidity, orientation, bounding-box width and height, maximum Feret diameter,
and the seven Hu moments computed from the image moments.

\section{Benchmark Implementation Details}\label{app:baselines}

We used the same stratified 70/15/15 split for all methods. This produced
980/210/210 train/validation/test samples for MPEG-7 and 984/211/211 for
BBBC010. Contours were extracted using the same \texttt{ConvertBinaryMasksToDMs}
procedure as also used for ShapeEmbed~\cite{Annafoix}: the longest object contour was retained and periodically spline-resampled to 64 points.

\subsection{Deterministic Baselines}

\paragraph*{\textit{Elliptical Fourier Descriptors}}
We used the \href{https://pyefd.readthedocs.io/en/latest/}{\texttt{pyefd}}
library to compute EFA. Descriptors were extracted
with \texttt{pyefd.elliptic\_fourier\_descriptors} using order 30 (order 15 for the interpretability MDS in Fig.~\ref{fig:mds_invariances}). Each harmonic contributes four coefficients $(a_n, b_n, c_n, d_n)$, yielding 120-dimensional (respectively 60-dimensional) feature vectors per shape. The coefficients were then normalized with \texttt{pyefd.normalize\_efd} to enforce invariance to translation, rotation, and scale.

\paragraph*{\textit{PF-SDM and PF-SDM + Skel}}
Both variants follow the definitions in Section~\ref{subsec:benchmarking} of the main text. For the BBBC010 data set, we excluded the GFP channel which reports worm viability (dead worms accumulate the dye and fluoresce brightly while live worms remain dark). This was to prevent clever-Hans behavior and keep the benchmarking shape-focused. The excluded PF-SDM\,+\,Int variant has the same runtime as PF-SDM\,+\,Skel and achieves an F1 test score of $0.86$.

\subsection{Deep-Learning Baselines}
For the deep-learning baselines, we limited training to 200 epochs as was also done for ShapeEmbed~\cite{Annafoix}.

\paragraph*{\textit{ShapeEmbed and ShapeEmbed + Size}}
We used the official ShapeEmbed~\cite{Annafoix} implementation
with a 128-dimensional latent space. Each input was the $64 \times 64$ Euclidean distance matrix computed from the resampled contour. For the ShapeEmbed + Size
variant, we concatenated the latent mean with ShapeEmbed's recovered original scale feature.

\paragraph*{\textit{MAE}}
We used the official \texttt{facebookresearch/mae}~\cite{MAE} implementation.
Binary masks were resized to $224 \times 224$. For pretrained MAE, we loaded the ImageNet-1K MAE checkpoints for the evaluated ViT configurations. For the data-trained MAE models, we initialized the MAE architecture from scratch and trained it on the training masks for 200 epochs using a batch
size of 16.

\paragraph*{\textit{SimCLR}}
We used the \texttt{sthalles/SimCLR}~\cite{simclr} implementation
with a ResNet18 backbone and
a 128-dimensional projection head. Binary masks were resized to $224 \times 224$ and positive pairs were generated with the default SimCLR augmentation pipeline: random resized crops, horizontal
flips, color jitter, random grayscale conversion, and Gaussian blur. The model
was trained on the training split for 200 epochs with batch size 256.

\paragraph*{\textit{O2VAE}}
We used the native O2VAE~\cite{Burgess2024} implementation
with binary masks as inputs, since the
benchmark focuses on shape. Masks were resized to $64 \times 64$. The O2VAE used an O(2)-equivariant CNN encoder and a 128-dimensional latent space. Training used batch size 256,
200 epochs, random rotations, and vertical flips.

\end{document}